\documentclass[sigconf,nonacm]{preprint} 

\usepackage{balance} %
\usepackage{tablefootnote}
\usepackage{moresize}
\usepackage{pgfplots}
\pgfplotsset{compat=1.18}
\usepackage{pgfplotstable}
\usepackage{siunitx}
\usepackage{multirow}
\usepackage{caption}
\usepackage{subcaption}
\usepackage{circuitikz}
\usetikzlibrary{positioning}
\usepackage{algorithm}
\usepackage{algorithmic}
\usepackage{xcolor}
\usepackage[table]{xcolor}
\usepackage{makecell}
\usepackage{xurl}
\usepackage{hyperref}
\usepackage{bbm}
\hypersetup{
    breaklinks=true
}
\usepackage{bbm}
\usepackage{enumitem}
\usepackage{footmisc}
\usepackage[inkscapelatex=false]{svg}

  \definecolor{cfapprox}{HTML}{2A78D6}
  \definecolor{cftight}{HTML}{EB6834}
  \definecolor{cfgumbel}{HTML}{1BAF7A}

\DeclareCaptionSubType[roman]{subfigure}
\pgfplotsset{tick label style={font=\footnotesize},
             label style={font=\footnotesize},
             legend style={font=\footnotesize}}

\theoremstyle{plain}
\newtheorem{theorem}{Theorem}[section]
\newtheorem{proposition}[theorem]{Proposition}

\theoremstyle{definition}
\newtheorem{definition}[theorem]{Definition}
\newtheorem{assumption}[theorem]{Assumption}
\theoremstyle{remark}

\DeclareMathOperator*{\argmax}{arg\,max}

\setcopyright{none}

\AtBeginDocument{%
  \fancypagestyle{firstpagestyle}{%
    \fancyhf{}%
    \fancyfoot[C]{\footnotesize\thepage}}%
  \fancypagestyle{standardpagestyle}{%
    \fancyhf{}%
    \fancyfoot[C]{\footnotesize\thepage}}%
  \pagestyle{standardpagestyle}%
}

\title[Robust Counterfactuals for Markov Decision Processes\\under Statistical and Causal Uncertainties]{Robust Counterfactuals for Markov Decision Processes\\under Statistical and Causal Uncertainties}

\author{Jessica Lally}
\affiliation{
  \institution{King's College London}
  \city{London}
  \country{United Kingdom}}
\email{jessica.lally@kcl.ac.uk}

\author{Milad Kazemi}
\affiliation{
  \institution{Durham University}
  \city{Durham}
  \country{United Kingdom}}
\email{milad.kazemi@durham.ac.uk}

\author{Nicola Paoletti}
\affiliation{
  \institution{King's College London}
  \city{London}
  \country{United Kingdom}}
\email{nicola.paoletti@kcl.ac.uk}

\newcommand{\interv}[5]{{#1}_{(#2\leftarrow#3,#4\leftarrow#5)}}

\newcommand{\probs}[1][]{%
\ifthenelse{\equal{#1}{}}{P_{\mdp}}{P_{\mdp(#1)}}%
}
\newcommand{\mdp}{\mathcal{M}}
\newcommand{\states}{\mathcal{S}}

\newcommand{\jl}[1]{#1}
\newcommand{\updated}[1]{#1}

\newcommand{\revised}[1]{{\color{black}#1}}

\begin{abstract}
Counterfactual reasoning (asking \emph{what would have happened, had a decision been different}) is a powerful tool for identifying causes of failure and explaining the behaviour of AI systems in sequential decision-making tasks. In Markov Decision Processes (MDPs), two sources of uncertainty complicate counterfactual inference: statistical uncertainty in the MDP transitions when the MDP is estimated from trajectory data, and uncertainty about the unknown causal model, as there are multiple structural causal models (SCMs) compatible with each MDP. Existing work addresses these two challenges in isolation. We address both simultaneously by introducing a framework for partial counterfactual inference in MDPs that is robust to both statistical and causal uncertainty. Our approach provides: 1) closed-form analytical bounds for the counterfactual probabilities that are valid across all SCMs entailed by a given interval MDP (IMDP); 2) a method to compute \revised{minimal robust interventions to achieve a desired counterfactual outcome} via pessimistic value iteration on the induced counterfactual IMDP. Since the latter is constructed using our analytical bounds, our approach remains efficient and scalable to realistic MDP environments. Experiments across multiple case studies demonstrate our method's robustness and practical utility for identifying minimal sets of actions that, if changed, would have led to improved outcomes.
\end{abstract}


\newcommand{\BibTeX}{\rm B\kern-.05em{\sc i\kern-.025em b}\kern-.08em\TeX}

\begin{document}

\maketitle

\section{Introduction}
Markov Decision Processes (MDPs) are a fundamental model for describing sequential decision-making processes under uncertainty, including reinforcement learning (RL) problems. Counterfactual reasoning for MDPs is central to policy evaluation, explanation, and accountability. Given an observed MDP path, i.e., a state-action sequence, one may ask: \textit{Would a different sequence of actions have led to a different outcome?} Such questions arise in numerous applications, such as generating counterfactual explanations \citep{oberst2019counterfactual, tsirtsis2024finding, tsirtsis2021counterfactual}, facilitating policy transfer \citep{killian2022counterfactually}, data augmentation with counterfactual paths \citep{lu2020sample, sun2024acamda}, as well as LLM counterfactuals \citep{chatzi2025counterfactual,pona2025abstract}.

However, two distinct sources of uncertainty complicate counterfactual reasoning for MDPs. First, counterfactuals in MDPs are generally not point-identified: many structural causal models (SCMs) can be compatible with the MDP, each yielding different counterfactual outcomes.  We call this \emph{causal uncertainty}. Yet, many approaches assume a fixed SCM, typically a Gumbel-max SCM \citep{benz2022counterfactual,
kazemi2024counterfactual,
killian2022counterfactually, noorbakhsh2022counterfactual,
tsirtsis2021counterfactual, zhu2020counterfactual}, potentially leading to invalid counterfactual conclusions. Second, the transition probabilities of an MDP are often estimated from finite trajectory data. Yet, most counterfactual analyses treat these point estimates as precise probabilities, ignoring \emph{statistical uncertainty} in the learned MDP. 

Recent work by \citet{lally2025robustcounterfactualinferencemarkov} addresses the causal uncertainty problem by bounding counterfactual probabilities across all SCMs consistent with the MDP. Also, there is a well-established line of research (e.g., \citep{suilen2022robust, meggendorfer2024odds, schnitzer2024certifiably}) on capturing statistical uncertainty by using uncertain MDPs, where each transition probability lies within a data-derived confidence set, most often an interval (yielding interval MDPs). 

However, there is no work to date that addresses both kinds of uncertainty, even though both matter for identifying robust \revised{interventions} and explanations. Consider a learned model for sepsis treatment. Suppose that, under a nominal transition estimate, changing the dosage at step $t$ prevents an adverse outcome. So we would conclude that action $a_t$ is \revised{a sufficient intervention to prevent that outcome,} 
but only if the transition probabilities were precise. Since these are learned from data and subject to sampling variability, there could be alternative models supported by the data where the adverse outcome would have occurred regardless of the dosage at $t$, i.e., \revised{where intervening on $a_t$ alone is not enough}. Thus, we seek \revised{counterfactual interventions} that are robust to statistical uncertainties, i.e., valid across all MDP models entailed by some data-derived confidence set. For a similar reason, since there are many SCM models compatible with an MDP, we want \revised{such interventions} to be robust to causal uncertainty, because each SCM can yield different counterfactual probabilities and thus different \revised{counterfactual policies that would have led to the patient's recovery.}

In this paper, we develop a framework for robust counterfactual reasoning in uncertain MDPs. Our contributions span two main areas:

\textit{1) Counterfactual bounds for uncertain MDPs.} We derive analytical bounds for the counterfactual transition probabilities of a given interval MDP (IMDP). \jl{Our bounds are robust to \textit{both} sources of uncertainty in this setting: they are guaranteed to include the true (unknown) counterfactual probabilities for every MDP $\mathcal{M}$ entailed by the IMDP, and every SCM compatible with $\mathcal{M}$.} To achieve this, we build on the work of \citet{lally2025robustcounterfactualinferencemarkov}, who derived bounds on counterfactual probabilities for standard (certain) MDPs, and extend it -- for the first time -- to the uncertain MDP setting. Crucially, without our analytical characterisation, computing robust counterfactual probabilities becomes prohibitive, as it entails solving an optimisation problem with a number of constraints exponential in the size of the MDP \citep{duarte2023, zaffalon2024,pmlr-v162-zhang22ab}. Finally, we prove that our analytical bounds are exact (i.e., tight) in all but one case. For this remaining case, we provide conservative (i.e., sound) bounds that remain empirically efficient. As a contribution of independent interest, we further derive conservative bounds for the special (parametric) case of Gumbel-max causal models.

\textit{2) \revised{Minimal robust counterfactual interventions.}} 
Building on these bounds, we define the notion of a \revised{minimal robust counterfactual intervention as an MDP policy that performs the smallest set of action changes to induce}  
a desired counterfactual outcome across all the causal models induced by statistical and causal uncertainty. In our MDP setting, w.l.o.g., we define outcomes as constraints over the (counter)factual path's expected return.
Given an IMDP estimated from data and an observed path, to derive such robust \revised{interventions}, our method first constructs a counterfactual IMDP by computing bounds for all its transitions; then,  \revised{interventions} are found as solutions to a robust value iteration procedure \citep{mathiesen2024intervalmdp} modified so that its value function keeps track of the allowed number of action changes. Finally, we evaluate our approach on a variety of MDP environments to investigate sources of failure in suboptimal RL policies, demonstrating that our method remains accurate and robust to uncertainty in both the MDP transitions and the causal model.

\section{Background}
In this section, we outline the background connecting Markov Decision Processes (MDPs), structural causal models (SCMs), and counterfactual inference.

\subsection{Markov Decision Processes}
MDPs model sequential decision-making under uncertainty: at each time $t$, an agent in state $s_t$ selects an action $a_t$ via a policy $\pi$, transitions to $s_{t+1} \sim P(\cdot \mid s_t, a_t)$, and receives reward $R(s_t,a_t)$. Formally, an MDP is a tuple $\mathcal{M}=(\mathcal{S},\mathcal{A},P,P_I,R)$, where $\mathcal{S}$ is the discrete \emph{state space}, $\mathcal{A}$ the set of \emph{actions}, $P: (\mathcal{S} \times \mathcal{A}) \rightarrow \Delta(\mathcal{S})$ the \emph{transition probability} function, $P_I \in \Delta(\mathcal{S})$ the \textit{initial state distribution}, and $R: (\mathcal{S} \times \mathcal{A}) \rightarrow \mathbb{R}$ the \emph{reward function}.

A (deterministic) \emph{policy} $\pi$ for $\mathcal{M}$ is a function $\pi: \mathcal{S} \rightarrow \mathcal{A}$. A path $\tau$ of $\mathcal{M}$ under policy $\pi$ is a sequence \revised{$\tau=(s_0, a_0, s_1, a_1, \ldots, s_{T-1}, a_{T-1}, s_T)$} where $T=|\tau|$ is the path length, $P_I(s_0)>0$, $a_t=\pi(s_t)$ for all $t=0,\ldots,T-1$, and $P(s_{t+1}\mid s_t, a_t)>0$ for all $t=0,\ldots,T-1$.

While MDPs can model a wide range of real-world decision-making problems \citep{feinberg2012handbook}, it is often challenging to determine the transition probability function ${P}$ exactly. To address this,  \textit{uncertain MDPs} allow us to explicitly account for uncertainty in the transition dynamics. In particular, \textit{interval MDPs} represent this uncertainty through lower and upper bounds on the transition probabilities.
\begin{definition}[Interval MDP \citep{coppola2024data}]
An Interval Markov Decision Process (IMDP) is a tuple $\mathcal{M_{\updownarrow}} = (\mathcal{S}, \mathcal{A}, {P_{\updownarrow}}, {P_I}, {R})$, where $\mathcal{S}$, $\mathcal{A}$, ${P_I}$, and ${R}$ are as in a standard MDP, and ${P_{\updownarrow}}$ is an uncertain transition function: for all $s,s'\in\mathcal{S}$ and $a\in\mathcal{A}$ there exist bounds $0 \leq P^{\it LB}_{\mathcal{M_{\updownarrow}}}(s' \mid s, a) \leq P^{\it UB}_{\mathcal{M_{\updownarrow}}}(s' \mid s, a) \leq 1$ such that
${P}_{\updownarrow}(s' \mid s, a) = [P^{\it LB}_{\mathcal{M_{\updownarrow}}}(s' \mid s, a), P^{\it UB}_{\mathcal{M_{\updownarrow}}}(s' \mid s, a)]$.
\end{definition}

\subsection{Counterfactual Inference}
Structural causal models (SCMs) \citep{halpern2005causes,pearl_2009} provide a mathematical framework for reasoning about causality and counterfactual inference. Formally, a SCM is a tuple $\mathcal{C}=(\mathbf{U},\mathbf{V},\mathcal{F},P(\mathbf{U}))$, where $\mathbf{V}$ is the set of endogenous (observed) variables, $\mathbf{U}$ is a set of exogenous (unobserved) variables, $P(\mathbf{U})$ is a joint distribution over each $U \in \mathbf{U}$, and $\mathcal{F}$ is a set of structural equations such that each $f_V \in \mathcal{F}$ determines the value of variable $V \in \mathbf{V}$ as a function of $U_V\in \mathbf{U}$ and of $V$'s endogenous parents $\mathbf{PA}_V \subseteq \mathbf{V}$. \jl{An \textit{intervention} (denoted $do(\mathbf{X}=\mathbf{x})$) modifies the SCM by fixing one or more variables $\mathbf{X} \subseteq \mathbf{V}$ to specific values $\mathbf{x}$ and severing their dependence on their respective parents $\mathbf{PA}_\mathbf{X}$ in the causal graph; combined with the original $P(\mathbf{U})$ this yields an \textit{interventional distribution} over the remaining variables. \textit{Counterfactual inference} follows the same procedure, except it infers a posterior $P(\mathbf{U} \mid \mathbf{v})$ over the values of the exogenous variables consistent with a given observation $\mathbf{v}$, and then applies the intervention $do(\mathbf{X}=\mathbf{x})$ using $P(\mathbf{U} \mid \mathbf{v})$ in place of $P(\mathbf{U})$, yielding a distribution over the counterfactual outcomes.}

\subsection{Counterfactual Inference in MDPs}\label{sec:CF_inf_mdp}
An MDP can be represented as an SCM with the causal graph of Figure \ref{fig:MDP_DAG} and the following structural equations:
\begin{equation}\label{eq:mdp_scm}
\begin{aligned}
        S_{t+1} = f(S_{t},A_{t}, U_{t}); \quad A_t = \pi(S_t); 
        \\ S_0 = f_I(U_I); \quad R_t = f_R(S_t, A_t).
\end{aligned}
\end{equation}
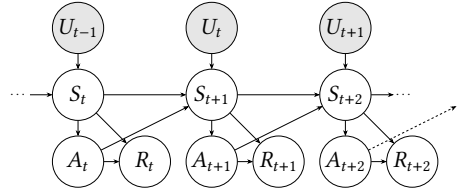
\begin{figure}
    \centering
    \resizebox{0.70\columnwidth}{!}{%
            \begin{circuitikz}
    \tikzstyle{every node}=[font=\HUGE]

    \node[circle, draw, minimum size=2.3cm] (S_tp1) at (3, 0) {$S_{t+1}$};
    \node[circle, draw, minimum size=2.3cm] (S_tp2) at (9, 0) {$S_{t+2}$};
    \node[circle, draw, minimum size=2.3cm] (S_t) at (-3, 0) {$S_t$};
    \node[circle, draw, minimum size=2.3cm] (A_t) at (-3, -3) {$A_t$};
    \node[circle, draw, minimum size=2.3cm] (A_tp1) at (3, -3) {$A_{t+1}$};
    \node[circle, draw, minimum size=2.3cm] (A_tp2) at (9, -3) {$A_{t+2}$};
    \node[circle, draw, minimum size=2.3cm] (R_t) at (0, -3) {$R_t$};
    \node[circle, draw, minimum size=2.3cm] (R_tp1) at (6, -3) {$R_{t+1}$};
    \node[circle, draw, minimum size=2.3cm] (R_tp2) at (12, -3) {$R_{t+2}$};
    \node[circle, draw, fill=gray!20, minimum size=2.3cm] (U_tm1) at (-3, 3) {$U_{t-1}$};
    \node[circle, draw, fill=gray!20, minimum size=2.3cm] (U_t) at (3, 3) {$U_{t}$};
    \node[circle, draw, fill=gray!20, minimum size=2.3cm] (U_tp1) at (9, 3) {$U_{t+1}$};
    \node at ([xshift=-1.5cm]S_t.west) {\Huge $\dots$};
    \node at ([xshift=1.5cm]S_tp2.east) {\Huge $\dots$};
    \draw[->, thick, >=Stealth]([xshift=-1cm]S_t.west) -- (S_t);
    \draw[->, thick, >=Stealth] (S_t) -- (S_tp1);
    \draw[->, thick, >=Stealth] (S_tp1) -- (S_tp2);
    \draw[->, thick, >=Stealth] (S_t) -- (A_t);
    \draw[->, thick, >=Stealth] (S_tp1) -- (A_tp1);
    \draw[->, thick, >=Stealth] (S_tp2) -- (A_tp2);
    \draw[->, thick, >=Stealth] (A_t) -- (S_tp1);
    \draw[->, thick, >=Stealth] (A_tp1) -- (S_tp2);
    \draw[->, thick, >=Stealth] (U_t) -- (S_tp1);
    \draw[->, thick, >=Stealth] (U_tm1) -- (S_t);
    \draw[->, thick, >=Stealth] (U_tp1) -- (S_tp2);
    \draw[->, thick, >=Stealth] (S_tp2) -- ([xshift=1cm]S_tp2.east);
    \draw[->, thick, dashed, >=Stealth] (A_tp2) -- ++(4.97,2.48);
    \draw[->, thick, >=Stealth] (S_t) -- (R_t);
    \draw[->, thick, >=Stealth] (A_t) -- (R_t);
    \draw[->, thick, >=Stealth] (S_tp1) -- (R_tp1);
    \draw[->, thick, >=Stealth] (A_tp1) -- (R_tp1);
    \draw[->, thick, >=Stealth] (S_tp2) -- (R_tp2);
    \draw[->, thick, >=Stealth] (A_tp2) -- (R_tp2);
\end{circuitikz}}
    \caption{MDP causal graph. White nodes are the endogenous variables; grey nodes are the exogenous variables.}
    \label{fig:MDP_DAG}
\end{figure}
\noindent The exogenous variables $U_t$ capture all sources of randomness in the MDP, rendering the structural equations deterministic (given $U_t$). In this setting, the interventional distribution \jl{$P(\interv{S_{t+1}}{S_t}{s}{A_t}{a} = s')$ corresponds to the MDP's transition probabilities $P(\cdot \mid s, a)$ and observations are paths $\tau = (s_0, a_0, \dots, s_{T-1}, a_{T-1}, s_T)$ of length $T$.}

Counterfactual inference is applied at the level of the individual timesteps: given the observed transition $s_t, a_t \rightarrow s_{t+1}$ at time $t$, \jl{we first infer the posterior distribution over the exogenous variable $U_t$ consistent with this observation, $P(U_t \mid S_t = s_t, A_t = a_t, S_{t+1} = s_{t+1})$. We then use this posterior to evaluate the counterfactual probability of reaching outcome $S_{t+1} = \tilde{s}’$, had we been in state $S_t = \tilde{s}$, and instead performed action $A_t = \tilde{a}$.} Formally, we define this counterfactual probability as $\tilde{P}_t(\tilde{s}' \mid \tilde{s}, \tilde{a}) = P(\interv{S_{t+1}}{S_t}{\tilde{s}}{A_t}{\tilde{a}} = \tilde{s}' \mid S_t = s_t, A_t = a_t, S_{t+1} = s_{t+1})$ where $\interv{S_{t+1}}{S_t}{\tilde{s}}{A_t}{\tilde{a}}$ denotes variable $S_{t+1}$ in the SCM after the intervention $S_t \gets \tilde{s}$ and $A_t \gets \tilde{a}$.

\jl{These single-timestep counterfactual probabilities can be combined across all timesteps to construct a non-stationary \textit{counterfactual} MDP \citep{tsirtsis2021counterfactual}, allowing us to reason about counterfactual outcomes over an entire trajectory.}

\begin{definition}[Counterfactual MDP (CFMDP)]
Given an MDP $\mathcal{M}$, \revised{a corresponding SCM $\mathcal{C}$ of the form~\eqref{eq:mdp_scm} (i.e., one whose interventional distribution matches the transition probabilities of $\mathcal{M}$)}, and an observed path $\tau = (s_0, a_0, s_1, a_1, \dots, s_{T-1}, a_{T-1}, s_T)$, the corresponding counterfactual MDP is a non-stationary MDP ${\tilde{\mathcal{M}}_{\tau}}=(\mathcal{S}, \mathcal{A}, \tilde{P}, \tilde{P}_{I}, {R})$ where $\tilde{P}_I(s_0) = 1$ and the time-dependent transition kernel is given, for any $t \in \{0,\ldots, |\tau|-1\}$, by $\tilde{P}_{t}(\tilde{s}' \mid \tilde{s}, \tilde{a}) = 
      P({S_{t+1}}_{(S_t=\tilde{s}, A_t=\tilde{a})} = \tilde{s}' \mid 
      S_t = s_t, A_t = a_t, S_{t+1} = s_{t+1}).$
\end{definition}

In addition to enabling predictions under hypothetical scenarios, counterfactual inference lets us identify the factors that were causally responsible for an observed outcome. \revised{In an MDP, an RL agent controls the outcome through its actions, so these factors correspond to the actions selected by the agent's policy. The CFMDP $\tilde{\mathcal{M}}_\tau$ lets us evaluate how the outcome would have changed under different actions, and therefore identify the fewest changes to the observed actions that would have led to a better outcome. However, constructing $\tilde{\mathcal{M}}_\tau$ requires knowing both the MDP $\mathcal{M}$ and its corresponding SCM $\mathcal{C}$, which is rarely the case; most often we face both statistical and causal uncertainty, which we address in the following sections.}

\section{Counterfactual Inference in Uncertain MDPs}
\jl{We first review how \citet{pmlr-v162-zhang22ab}’s partial counterfactual inference approach applies to MDPs, and how \citet{lally2025robustcounterfactualinferencemarkov} build on this to derive closed-form bounds on counterfactual MDP probabilities.} We then extend these bounds to IMDPs, forming the basis of our robust counterfactual inference framework.

\subsection{Canonical SCMs}
\jl{To address the problem of causal model unidentifiability, \citet{pmlr-v162-zhang22ab} introduce a family of \textit{canonical SCMs} that can simultaneously capture all possible counterfactual probability distributions across all compatible causal models. They show how to optimise over these canonical SCMs to find the SCMs that minimise and maximise counterfactual probabilities, yielding tight bounds.}

\jl{\citet{lally2025robustcounterfactualinferencemarkov} construct canonical SCMs for MDPs. Specifically, given a known MDP $\mathcal{M}$ (with a known transition function $P_\mathcal{M}$), the first step is to associate each value of the exogenous variable $u_t \in U_t$ with a unique deterministic mapping $f: \mathcal{S} \times \mathcal{A} \times \mathcal{U} \rightarrow \mathcal{S}$ that defines the next state for every state-action pair. Since there are finitely many such mappings\footnote{Wlog, \citet{pmlr-v162-zhang22ab} prove that for SCMs with categorical endogenous variables, the exogenous distributions can be represented by discrete, finite variables.}, this gives $|U_t| = {|\mathcal{S}|}^{|\mathcal{S}| \times |\mathcal{A}|}$. This mapping is represented by an indicator variable $\mu \in \{0,1\}^{|\mathcal{S}| \times |\mathcal{A}| \times |{U_t}| \times |\mathcal{S}|}$ such that, for any $s, a, u_t, s'$:}
\[
\mu_{s, a, u_t, s'} = 
\begin{cases}
    1 & \text{if $f(s, a, u_t) = s'$}\\
    0 & \text{otherwise.}\\
\end{cases}
\]
\jl{That is, $\mu_{s, a, u_t, s'} = 1$ when the mapping associated with $u_t$ leads to the next state $s'$ from state-action pair $(s, a)$}. Next, we define a vector $\theta \in \mathbb{R}^{|{U_t}|}$, where each $\theta_{u_t}$ represents the probability $P(U_t = u_t)$. Each instantiation of $\theta$ defines a unique SCM. Given an observed transition $s_t, a_t \rightarrow s_{t+1}$, we can express the counterfactual probability $\tilde{P}_{\mathcal{M},t}(s' \mid s, a)$ of any transition $s, a \rightarrow s'$ in terms of $\mu$ and $\theta$ as follows:
\begin{equation}\label{eq:cf_probs}
    \tilde{P}_{\mathcal{M},t}(s' \mid s, a) = \frac{\sum_{u_t = 1}^{|U_t|} \mu_{s, a, u_t, s'} \cdot \mu_{s_t, a_t, u_t, s_{t+1}} \cdot \theta_{u_t}}{P_{\mathcal{M}}(s_{t+1} \mid s_t, a_t)}\,.
\end{equation}
\jl{To obtain bounds $[\tilde{P}^{\it LB}_{\mathcal{M},t}(s' \mid s, a),\tilde{P}^{\it UB}_{\mathcal{M},t}(s' \mid s, a)]$ on the counterfactual transition probabilities, we optimise over all values of $\theta$ (note that $\mu$ is fixed and is not optimised over) consistent with the observed transition and interventional data (i.e., the MDP transition probabilities). This results in the following optimisation problem (omitting the constant denominator in \eqref{eq:cf_probs} from the objective):}
\begin{align}
\label{eq: optimisation objective}
\min_{\theta} / \max_{\theta} \textstyle\sum_{u_t=1}^{|U_t|} \mu_{s,a,u_t,s'} \cdot \mu_{s_t, a_t, u_t, s_{t+1}} \cdot \theta_{u_t},\\
\label{eq: interventional constraint}
\text{s.t.}\ \textstyle\sum_{u_t=1}^{|U_t|} \mu_{\hat s,\hat a,u_t,\hat s'}\cdot\theta_{u_t} = P_{\mathcal M}(\hat s'\mid \hat s,\hat a),\ \forall \hat s,\hat a,\hat s',\\
\label{eq: u constraints}
\forall u_t, 0 \leq \theta_{u_t} \leq 1 \text{ and } \textstyle\sum_{u_t=1}^{|U_t|} \theta_{u_t} = 1\,.
\end{align}
\subsection{Closed-form Bounds for MDPs}
\label{sec: closed-form bounds mdps}
\citet{lally2025robustcounterfactualinferencemarkov} prove that, in MDP settings, this optimisation problem reduces to exact, closed-form expressions. These bounds also incorporate two key assumptions, counterfactual stability and counterfactual monotonicity, to avoid unintuitive counterfactual probabilities. Without these assumptions, the optimisation in \eqref{eq: optimisation objective}-\eqref{eq: u constraints} tends to produce loose or trivial $[0, 1]$ bounds, which compound over multiple transitions in sequential decision-making and reduce the usefulness of counterfactual inference. 

\jl{Counterfactual stability is a standard assumption in the literature \citep{oberst2019counterfactual, tsirtsis2021counterfactual} and states that the outcome under an intervention changes only if the relative probability of some other outcome increases relative to the observed outcome.} Counterfactual monotonicity ensures that observing an outcome cannot make it less likely in the counterfactual setting, and that unobserved (but possible) outcomes cannot become more likely. 
Formal definitions and the resulting bounds are in Appendix \ref{app: assumptions and bounds}.

\subsection{Extending the Bounds to IMDPs}
To extend these counterfactual probability bounds to IMDPs, we generalise the optimisation problem in \eqref{eq: optimisation objective}-\eqref{eq: u constraints} as follows. Let $\mathcal{M_\updownarrow}$ denote an arbitrary IMDP. For a given counterfactual transition $s, a \rightarrow s'$ and observed transition $s_t, a_t \rightarrow s_{t+1}$, the upper and lower bounds $\tilde{P}_{\mathcal{M}_\updownarrow,t}^{\it UB}(s' \mid s, a)$ and $\tilde{P}_{\mathcal{M}_\updownarrow,t}^{\it LB}(s' \mid s, a)$ are obtained by solving
\begin{equation}
\begin{aligned}
&\underset{\mathcal{M} \in \mathcal{M_\updownarrow},\, \theta}{\min / \max}\left( \frac{\sum_{u_t = 1}^{|U_t|} \mu_{s, a, u_t, s'} \cdot \mu_{s_t, a_t, u_t, s_{t+1}} \cdot \theta_{u_t}}{P_{\mathcal{M}}(s_{t+1}\mid s_t, a_t)} \right) \\
\end{aligned}    
\label{eq: imdp optimisation}
\end{equation}
\jl{subject to constraints \eqref{eq: interventional constraint} and \eqref{eq: u constraints}, and crucially ensuring that $\mathcal{M}$ is a valid MDP within the IMDP's uncertainty set}, i.e., 
$\forall s, a, \sum_{s' \in \mathcal{S}}P_\mathcal{M}(s' \mid s, a) = 1$.
Since closed-form expressions for the counterfactual bounds are already available for a fixed MDP $\mathcal{M}$, the optimisation over the IMDP reduces to optimising these bounds over all admissible MDPs:
\begin{align}
\tilde{P}^{UB}_{\mathcal{M}_\updownarrow,t}(s' \mid s, a) = {\max}_{\mathcal{M} \in \mathcal{M_\updownarrow}} \left( \tilde{P}^{UB}_{\mathcal{M},t}(s' \mid s, a) \right),
\label{eq: optimisation ub}\\
\tilde{P}^{LB}_{\mathcal{M}_\updownarrow,t}(s' \mid s, a) = {\min}_{\mathcal{M} \in \mathcal{M_\updownarrow}} \left( \tilde{P}^{LB}_{\mathcal{M},t}(s' \mid s, a) \right).
\label{eq: optimisation lb}
\end{align}
This formulation enables us to derive closed-form expressions for the counterfactual probability bounds. \jl{Without it, computing these bounds would require solving \eqref{eq: imdp optimisation} directly: a joint optimisation over admissible MDPs $\mathcal{M} \in \mathcal{M}_\updownarrow$, and over $\theta$ (whose cardinality is exponential in the size of the state and action spaces)}. \jl{However, deriving these closed-form expressions requires more than simply optimising the closed-form MDP bounds of \citet{lally2025robustcounterfactualinferencemarkov} independently over each probability interval: the resulting extremal values must also be simultaneously achievable by a single, valid MDP $\mathcal{M} \in \mathcal{M}_\updownarrow$, which makes extending their bounds to IMDPs non-trivial.}

\jl{The exact closed-form expression depends on whether the counterfactual state-action pair $(s,a)$ and the observed state-action pair have disjoint or overlapping support: in the IMDP setting, we say $(s, a)$ has disjoint support from the observed $(s_t,a_t)$ if their supports are disjoint for \emph{every} admissible MDP $\mathcal{M} \in \mathcal{M}_\updownarrow$; otherwise, they have overlapping support. 
The resulting closed-form bounds are summarised in Table~\ref{tab:upper-bounds}, using $p^{\it LB}_{s, a}(s')$ as shorthand for $P^{\it LB}_{\mathcal{M}_\updownarrow}(s' \mid s, a)$ and $p^{\it UB}_{s, a}(s')$ for $P^{UB}_{\mathcal{M}_\updownarrow}(s' \mid s, a)$, for all $s, a, s'$; full theorem
statements and proofs are given in Appendix~\ref{app: proofs}.}
\begin{table*}[h!]
\centering
\small
\setlength{\tabcolsep}{3pt}
\renewcommand{\arraystretch}{1.25}
\caption{Closed-form counterfactual probability upper and lower bounds. For $s' \neq s_{t+1}$ in the overlapping case, the upper and lower bound conditions are each applied in order, where \emph{CS always holds} if $p^{\it LB}_{s_t,a_t}(s')>0$ and $\frac{p^{\it LB}_{s,a}(s_{t+1})}{p^{\it UB}_{s_t,a_t}(s_{t+1})}\geq\frac{p^{\it UB}_{s,a}(s')}{p^{\it LB}_{s_t,a_t}(s')}$, and \emph{CS sometimes holds} if $p^{\it UB}_{s_t,a_t}(s')>0$ and $\frac{p^{\it UB}_{s,a}(s_{t+1})}{p^{\it LB}_{s_t,a_t}(s_{t+1})}\geq\frac{p^{\it LB}_{s,a}(s')}{p^{\it UB}_{s_t,a_t}(s')}$. $\mathcal{P}_{\it sum}=\sum_{s'\neq s_{t+1}}\tilde{P}^{\it UB}_{\mathcal{M}_\updownarrow,t}(s'\mid s,a)$ for the closed-form but conservative lower bound; for the tight, optimisation-based bound, $\mathcal{P}_{\it sum}=\max_{\mathcal{M}\in\mathcal{M}_\updownarrow} \textstyle\sum_{s'' \in \mathcal{S}\setminus\{s'\}} \tilde{P}_{\mathcal{M},t}^{\it UB}(s'' \mid s, a)$.}
\begin{tabular}{@{}p{0.12\textwidth} p{0.12\textwidth} p{0.33\textwidth} @{\hspace{8pt}} p{0.12\textwidth} p{0.25\textwidth}@{}}
\toprule
\textbf{Theorem} & \textbf{UB condition} & \textbf{Upper bound} $\tilde{P}^{\it UB}_{\mathcal{M}_\updownarrow,t}(s'\mid s,a)$ & \textbf{LB condition} & \textbf{Lower bound} $\tilde{P}^{\it LB}_{\mathcal{M}_\updownarrow,t}(s'\mid s,a)$ \\
\midrule
Thm \ref{theorem: observed imdp}: $(s,a)=(s_t,a_t)$
& n/a
& $1$ if $s' = s_{t+1}$, $0$ otherwise
& n/a
& $1$ if $s' = s_{t+1}$, $0$ otherwise \\
\midrule
Thms \ref{theorem: upper disjoint imdp}, \ref{theorem: lower disjoint imdp}: $(s,a)\neq(s_t,a_t)$, supports disjoint
& n/a
& $\min\Big\{1,\, \dfrac{p^{\it UB}_{s, a}(s')}{p^{\it LB}_{s_t, a_t}(s_{t+1})}\Big\}$
& n/a
& $\max\Big\{0,\, \dfrac{p^{\it LB}_{s, a}(s') - \big(1-p^{\it LB}_{s_t, a_t}(s_{t+1})\big)}{p^{\it LB}_{s_t, a_t}(s_{t+1})}\Big\}$ \\
\midrule
\multirow{5}{0.12\textwidth}{Thms \ref{theorem: upper overlapping imdp}, \ref{theorem: lower overlapping imdp}: $(s,a)\neq(s_t,a_t)$, supports overlap}
& $s'=s_{t+1}$
& $\min\Big\{1,\, \dfrac{p^{\it UB}_{s, a}(s_{t+1})}{p^{\it LB}_{s_t, a_t}(s_{t+1})}\Big\}$
& $s'=s_{t+1}$
& $\max\big(p^{\it LB}_{s,a}(s'),\,1-\mathcal{P}_{\it sum}\big)$ \\
& CS always holds
& $0$
& CS sometimes holds
& $0$ \\
& $p^{\it LB}_{s_t,a_t}(s')>0$
& $\min\big(p^{\it UB}_{s,a}(s'),\,1-p^{\it LB}_{s,a}(s_{t+1})\big)$
& otherwise
& $\max\big(0,\,1-\mathcal{P}_{\it sum}\big)$ \\
& $p^{\it UB}_{s_t,a_t}(s')=0$
& $\min\Big(1-p^{\it LB}_{s,a}(s_{t+1}),\,\dfrac{p^{\it UB}_{s,a}(s')}{p^{\it LB}_{s_t,a_t}(s_{t+1})}\Big)$
& & \\
& otherwise
& $\min\Big(1-p^{\it LB}_{s,a}(s_{t+1}),$\newline
  $\phantom{\min\Big(}\dfrac{p^{\it UB}_{s,a}(s')}{\max\Big(\begin{array}{@{}l@{}}1-\sum_{s''\neq s',s_{t+1}}p^{\it UB}_{s_t,a_t}(s''),\\ p^{\it LB}_{s_t,a_t}(s_{t+1})\end{array}\Big)}\Big)$
& & \\
\bottomrule
\end{tabular}
\label{tab:bounds}
\label{tab:upper-bounds}
\label{tab:lower-bounds}
\end{table*}

\jl{\paragraph{Proof Sketch}
Consider an observed transition $s_t, a_t \rightarrow s_{t+1}$ and counterfactual transition $s, a \rightarrow s’$. From ~\eqref{eq:cf_probs}, for a fixed MDP $\mathcal{M}$ and exogenous distribution $\theta$, we see that the counterfactual probability $\tilde{P}_{\mathcal{M}, t}(s’ \mid s, a)$ depends only on the subset of exogenous values $u_t$ consistent with both transitions, since $P_{\mathcal{M}}(s_{t+1} \mid s_t, a_t)$, the denominator in~\eqref{eq:cf_probs}, is fixed. Therefore, the optimisation involves identifying the $\theta$ that maximises/minimises the following sum: 
\[ 
    \sum_{u_t = 1}^{|U_t|} \mu_{s, a, u_t, s'} \cdot \mu_{s_t, a_t, u_t, s_{t+1}} \cdot \theta_{u_t} = \sum_{\substack{u_t \in U_t \\f(s_t, a_t, u_t) = s_{t+1} \\ f(s, a, u_t) = s'}} \theta_{u_t} \,.
\] 

Importantly, \citet{lally2025robustcounterfactualinferencemarkov} prove that the bounds for each $(s, a)$ depend only on the probability distribution of this state-action pair and the observed state-action pair $(s_t, a_t)$, so we can decompose the optimisation problem. Since the bounds proven by \citet{lally2025robustcounterfactualinferencemarkov} are tight, extending these bounds from MDPs to IMDPs \eqref{eq: imdp optimisation} reduces to \eqref{eq: optimisation ub} and \eqref{eq: optimisation lb}, where we aim to identify, for each transition, the max/min possible counterfactual upper/lower bound across all admissible MDPs within the IMDP. In the proofs, we search over the interval probabilities for $(s, a)$ and $(s_t, a_t)$ to obtain these bounds, while ensuring that a valid MDP exists (i.e., where the probability distribution for each state-action pair sums to 1) that results in these bounds. Intuitively, each case reflects a different way probability mass can be distributed between the transitions from $(s, a)$ and $(s_t, a_t)$ under this constraint: when $(s,a)=(s_t,a_t)$, the bounds are immediate (the same $(s_t,a_t)$ under the same context $u_t$ must yield the same outcome); when the supports are disjoint, the counterfactual stability and counterfactual monotonicity assumptions are vacuously satisfied, and do not constrain how $(s,a)$ and $(s_t,a_t)$ behave under the same $u_t$, so the bound depends only on their marginal probabilities; and when the supports are overlapping, these assumptions constrain how probability mass can be distributed across the shared support of $(s,a)$ and $(s_t,a_t)$.
}

\jl{
\begin{proposition}[Tightness of the Counterfactual IMDP Bounds]
    \label{prop:tightness}
The upper bounds $\tilde{P}^{UB}_{\mathcal{M}_\updownarrow,t}(s' \mid s, a)$ are tight for all $s, a, s'$. The lower bounds $\tilde{P}^{LB}_{\mathcal{M}_\updownarrow,t}(s' \mid s, a)$ are tight for all $s, a, s'$, except when $(s,a) \neq (s_t,a_t)$, the counterfactual and observed transitions have overlapping support, and the counterfactual stability constraint does not force the bound to $0$ for all admissible MDPs $\mathcal{M} \in \mathcal{M}_\updownarrow$, i.e., Cases 1 and 3 of Theorem~\ref{theorem: lower overlapping imdp}.
\end{proposition}
}
\jl{\paragraph{Remark} In every case where a tight bound can be obtained (Theorems \ref{theorem: observed imdp}, \ref{theorem: upper disjoint imdp}, \ref{theorem: lower disjoint imdp}, \ref{theorem: upper overlapping imdp}, and Case 2 of \ref{theorem: lower overlapping imdp}, each closed-form bound can be expressed as a function of a small, fixed set of individual transition probabilities. This lets us explicitly construct a single admissible MDP that sets each of these probabilities to its IMDP-level extremal (upper or lower) value, thereby minimising or maximising the overall expression. However, Cases 1 and 3 of \ref{theorem: lower overlapping imdp} differ structurally, as the bound depends on the sum of the upper-bound probabilities of alternative outcomes $\sum_{s'' \in \mathcal{S}\setminus\{s'\}}{\tilde{P}_{\mathcal{M},t}^{\it UB}(s'' \mid s, a)}$, each term of which is itself the result of a separate optimisation over $\mathcal{M}_\updownarrow$. This leaves us with two options: (i) solving the nested optimisation when necessary to obtain the true maximum of $\sum_{s'' \in \mathcal{S}\setminus\{s'\}}{\tilde{P}_{\mathcal{M},t}^{\it UB}(s'' \mid s, a)}$ achievable by a single admissible MDP, giving a tight bound, or (ii) computing a conservative lower bound by replacing this sum with the sum of the individually-optimised IMDP-level upper bounds $\sum_{s'' \in \mathcal{S}\setminus\{s'\}}{\tilde{P}_{\mathcal{M}_\updownarrow,t}^{\it UB}(s'' \mid s, a)}$, which need not be jointly achievable. We compare these two approaches in our experiments.}

\paragraph{Bounds for Gumbel-max SCMs}
While this is not the main focus of our work, we also extend the Gumbel-max SCM counterfactual inference approach \citep{oberst2019counterfactual, tsirtsis2021counterfactual, killian2022counterfactually} from MDPs to IMDPs and, to our knowledge, provide the first bounds on Gumbel-max counterfactual probabilities for IMDPs. \jl{Although these bounds are not tight in general, they may be of independent interest given the popularity of Gumbel-max SCMs; we also use them in our experimental evaluation (Section~\ref{sec:eval}) as a baseline to assess the robustness of our approach. The bounds and their proof are given in Appendix~\ref{app: gumbel max}.}

\section{Methodology}\label{sec:methods}
\revised{We first describe how to derive an IMDP from paths sampled from the unknown MDP, and how to construct a counterfactual IMDP (CFIMDP) for a given observed path.} We then formalise the notion of \revised{minimal robust counterfactual interventions} and show how to compute them via robust value iteration over the CFIMDP.

\subsection{Learning IMDPs from Data}
\label{sec: imdp construction}
We begin by collecting observed trajectories of length $|\tau|$ (an\\ environment-specific hyperparameter) from either the initial state, if defined, or any non-sink state. Following the standard approach \citep{suilen2022robust, meggendorfer2024odds, schnitzer2024certifiably}, we estimate the IMDP by computing a confidence interval for each transition probability from these observations, assuming that the MDP's topology is known \jl{(the so-called \textit{grey-box} setting)\footnote{The \textit{black-box} alternative, where the topology is instead inferred from data by assuming a lower bound on the minimum transition probability \citep{meggendorfer2024odds}, is also compatible with our framework.}}. We use the Clopper-Pearson method \citep{clopper1934use}, with an overall confidence level $1-\alpha$ distributed evenly across all nondeterministic transitions (see Appendix \ref{app: confidence proof} for more details). \revised{We assign a minimum probability $\epsilon$ to all possible transitions ($\epsilon=0.01$, or $10^{-4}$ for Sepsis, whose smallest transition probability is approximately $10^{-4}$) to ensure that observed transitions have strictly positive lower bounds, as required for valid counterfactual inference.} \revised{We note our framework is agnostic to the estimation method: more sample-efficient methods than Clopper-Pearson may yield tighter IMDP bounds from the same data, and hence tighter CFIMDP bounds.} 

\subsection{Constructing the CFIMDP}
Given the counterfactual probability bounds in Table~\ref{tab:upper-bounds}, we can compute a non-stationary counterfactual interval MDP (CFIMDP) for any given IMDP $\mathcal{M}_\updownarrow$ and observed path $\tau$. For each transition observed in $\tau$, we compute the counterfactual probability bounds for all possible transitions in the IMDP. These bounds are then combined to form a non-stationary CFIMDP by concatenating the time-dependent counterfactual transition probability bounds across time steps. Formally, a CFIMDP is defined as follows:

\begin{definition}[Counterfactual Interval MDP (CFIMDP)\label{def:cfimdp} \citep{lally2025robustcounterfactualinferencemarkov}]
Given an observed path $\tau$ and IMDP $\mathcal{M}_\updownarrow$, the counterfactual probability for each transition in the CFIMDP ${\tilde{\mathcal{M}}^{\updownarrow}_{\tau}}$ is defined, for $t=0,\ldots,|\tau|-1$, as $\tilde{P}^{\updownarrow}_{t}({s}' \mid {s}, {a}) = [\tilde{P}_{t}^{\it LB}({s'} \mid {s}, {a}), \tilde{P}_{t}^{\it UB}({s'} \mid {s}, {a})]$.
\end{definition}
\jl{\begin{proposition}[Confidence propagation]
\label{prop: confidence}
Let $\mathcal{M}_\updownarrow$ be a data-driven IMDP, which contains the true MDP with probability at least $1-\alpha$. Then the CFIMDP derived from $\mathcal{M}_\updownarrow$ contains the true CFMDP with probability at least $1-\alpha$ (proof in Appendix~\ref{app: confidence proof}).
\end{proposition}}

\subsection{Minimal Robust Counterfactual Interventions}
The CFIMDP we just introduced allows us to systematically reason about counterfactual outcomes. Specifically, given an observed trajectory whose cumulative reward falls below a threshold $\lambda$, we seek an intervention that guarantees an \revised{expected counterfactual cumulative reward $\geq \lambda$}. \revised{Using the CFIMDP, we can identify interventions that would have improved on the factual outcome} in a \emph{robust} way w.r.t.\ all models induced by statistical and causal uncertainties. Moreover, we are interested in \emph{minimal} interventions, i.e., policies that require the fewest action changes to achieve this. Given our RL context, we focus on the expected cumulative reward. Properties like reach-avoid can be captured this way via a suitable Bellman recursion \citep{ABATE20082724}; more complex temporal logic properties can be encoded as reachability by taking the product of the CFIMDP with a reward machine \citep{icarte2022reward}, which
we leave for future work.

Let $\tau = (s_0, a_0, \dots, s_{|\tau|-1}, a_{|\tau|-1}, s_{|\tau|})$ be an observed path,
and let $G(\tau) = \sum_{t=0}^{|\tau|-1} R(s_t, a_t)$ denote its cumulative reward.
\revised{A \emph{counterfactual policy} $\pi$ chooses an action $\pi(t,s)$ for each time step
$t$ and state $s$. We write $\tilde{\mathcal M}_\tau\in\tilde{\mathcal{M}}^{\updownarrow}_{\tau}$
for any CFMDP within the CFIMDP, i.e., any non-stationary MDP whose transition
probabilities satisfy
$\tilde P^{\it LB}_t(s'\mid s,a)\le\tilde P_t(s'\mid s,a)\le\tilde P^{\it UB}_t(s'\mid s,a)$
for all $t$, $s$, $a$ and $s'$. Following $\pi$ in $\tilde{\mathcal M}_\tau$ induces a distribution
$\tilde{\mathcal M}_\tau(\pi)$ over counterfactual paths
$\tilde\tau=(\tilde s_0,\tilde a_0,\ldots,\tilde s_{|\tau|})$, and we define the worst-case
expected counterfactual cumulative reward of $\pi$ from the initial state $s_0$ as
\begin{equation}\label{eq:wc_v}
V^{\bot,\pi}(s_0) = \min_{ \tilde{\mathcal{M}}_{\tau} \in \tilde{\mathcal{M}}^{\updownarrow}_{\tau} }
\mathbb{E}_{\tilde{\tau} \sim \tilde{\mathcal{M}}_{\tau}(\pi) }
\big[G(\tilde{\tau}) \mid \tilde s_0 = s_0\big].
\end{equation}
The set $\tilde{\mathcal{M}}^{\updownarrow}_{\tau}$ contains the CFMDP of every MDP in the
IMDP under every corresponding SCM. Since the CFIMDP bounds each counterfactual transition
probability separately, it may also contain CFMDPs that combine probabilities induced by
different MDPs or SCMs, and so do not correspond to any single causal model compatible with the IMDP. As is standard
for IMDPs~\citep{GIVAN200071}, taking the worst case over this larger set yields
conservative, but efficiently computable, guarantees.}

\revised{For $k \in \{0,\dots,|\tau|\}$, we denote with $\Pi_{\tau,k}$ the set of counterfactual policies that change no more than $k$ actions from $\tau$, i.e.,
$$ \Pi_{\tau,k} =\left\{ \pi \;\middle|\; \sum_{t=0}^{|\tau|-1} \mathbbm{1}(\tilde a_t \neq a_t) \leq k,
\ \forall\tilde{\mathcal M}_\tau\in\tilde{\mathcal{M}}^{\updownarrow}_{\tau}
\text{ and } 
\tilde\tau\in \mathbf{supp}(\tilde{\mathcal M}_\tau(\pi)) \right\},$$
where $\mathbbm{1}\{\cdot\}$ denotes the indicator function, $\tilde a_t = \pi(t,\tilde s_t)$ is the action taken by $\pi$ at time $t$ along
$\tilde\tau$ and $\mathbf{supp}(\tilde{\mathcal M}_\tau(\pi))$ denotes the set of admissible paths of $\tilde{\mathcal M}_\tau(\pi)$.}

\revised{\begin{definition}[Minimal Robust Counterfactual Intervention]\label{def:mrci}
A counterfactual policy $\pi^*$ is a \emph{minimal robust counterfactual intervention} for a
path $\tau$ and threshold $\lambda$ if, for some $k^* \in \{0,\ldots,|\tau|\}$:
\begin{enumerate}
    \item $\pi^* \in \Pi_{\tau,k^*}$ and $V^{\bot,\pi^*}(s_0)\geq \lambda$, and
    \item for all $k < k^*$, $\max_{\pi \in \Pi_{\tau,k}} V^{\bot,\pi}(s_0) < \lambda$.
\end{enumerate}
\end{definition}}

\noindent \revised{Conversely, if the observed cumulative reward is at least $\lambda$, we can
identify the minimal interventions that would have robustly led to a cumulative reward
below $\lambda$, i.e., the observed actions responsible for the success, by reversing the
inequalities and replacing the $V^{\bot,\pi}$ with the corresponding
best-case (optimistic) value $V^{\top,\pi}$.}

\jl{\paragraph{\revised{Computing minimal robust counterfactual interventions.}} \revised{We compute minimal robust counterfactual interventions via pessimistic value
iteration~\citep{mathiesen2024intervalmdp}, which maximises the expected cumulative reward
under the worst-case CFMDP within the CFIMDP. To restrict the search to
$\pi \in \Pi_{\tau,k}$, we augment the state with a counter $k$ that tracks the remaining allowed number of
action changes, and define
\begin{multline}\label{eq:bellman_budget}
V^*(s,t,k) = \max_{\pi} \;\min_{\tilde{P}_t \in \tilde{P}^{\updownarrow}_{t}}
\mathbb{E}_{s' \sim \tilde{P}_t(\cdot \mid s, \pi(t,s))} \Big[R\big(s,\pi(t,s)\big) \\
+ V^*\big(s',\,t+1,\,k - \mathbbm{1}\{\pi(t,s) \neq a_t\}\big)\Big],
\end{multline}
where $V^*(s,|\tau|,k)=0$ for all $s$ and $k$, and, for $k=0$, the maximum is restricted to $\pi(t,s)=a_t$. Since the CFIMDP selects transition probabilities independently for each $(t,s,a)$, the worst case can be computed one step at a time~\citep{GIVAN200071}, so $V^*(s_0,0,k)=\max_{\pi\in\Pi_{\tau,k}}V^{\bot,\pi}(s_0)$. Let $k^*$ be the smallest $k$ with $V^*(s_0,0,k)\ge\lambda$, and $\pi^*_{k^*}$ the corresponding maximising policy. Then $\pi^*_{k^*}\in\Pi_{\tau,k^*}$ achieves $V^{\bot,\pi^*_{k^*}}(s_0)\ge\lambda$, and no policy that changes fewer actions does, so $\pi^*_{k^*}$ is a minimal robust counterfactual intervention (see Appendix~\ref{app: algorithm} for the full algorithm). With the conservative lower bounds, $\pi^*_{k^*}$ is still guaranteed to achieve $V^{\bot,\pi^*_{k^*}}(s_0)\ge\lambda$, but $k^*$ may be larger than with the tight bounds.}}
\revised{\paragraph{Relation to actual causation.} Minimal robust counterfactual interventions are closely related to the notion of actual causes~\citep{halpern2016actual}. Informally, a set of observed actions is an \textit{actual cause} of an outcome if changing those actions, while keeping all other actions at their observed values, would have changed the counterfactual outcome, and no strict subset of these actions has this property. Identifying an actual cause requires knowing the SCM and the context (the exogenous noise values), which together determine a single counterfactual path. In our setting, neither is known: we only know that the true CFMDP lies within the CFIMDP, and each CFMDP induces a distribution over counterfactual paths. We therefore require the change in outcome to hold robustly and in expectation, making our interventions a robust, policy-based analogue of actual causes.}

\section{Evaluation}\label{sec:eval}
\updated{We evaluate our approach along five axes: (1) its accuracy in off-policy evaluation tasks; (2) how its policies compare against those obtained from the exact CFIMDP (i.e., the CFIMDP derived from the true MDP); (3) its robustness relative to the Gumbel-max approach; (4) \revised{how well it identifies minimal action changes to achieve a desired counterfactual value}; and (5) the trade-off of using conservative vs.\ tight lower bounds.} We evaluate our approach on a range of MDP tasks, ranging from simple models (GridWorld and Frozen Lake) to more complex clinical and navigation tasks (Sepsis and Aircraft). Full descriptions of the MDP environments and system details can be found in the Appendix \ref{app: environments}. The code is available at:
\url{https://github.com/jessicalally/robust-cf-inference-beyond-mdps}.
\begin{table}[t]
\centering
  \caption{Average lower ($V^{\bot,\pi^*}(s_0)$) and upper ($V^{\top,\pi^*}(s_0)$) bounds on the counterfactual value of $\pi^*$ over $100$ randomly sampled observed paths, under the exact CFIMDP and data-driven CFIMDPs whose IMDPs are learnt from $N$ simulated paths. The true interventional value is $V^{\pi^*}(s_0)=333$.}%
\label{tab:cf_correctness_simulated_gridworld_0.9}

\begin{tabular}{lcc}
\toprule
& \textbf{Lower Bound} & \textbf{Upper Bound} \\
\midrule
Exact Bounds & $249 \pm 118$ & $387 \pm 39$ \\
\midrule
\multicolumn{3}{l}{\textbf{Data-driven Bounds}} \\
$N=10{,}000$       & $183 \pm 104$    & $398 \pm 36$ \\
$N=100{,}000$      & $228 \pm 114$    & $392 \pm 37$ \\
$N=1{,}000{,}000$  & $241 \pm 117$    & $389 \pm 38$ \\
\bottomrule
\end{tabular}
\end{table}
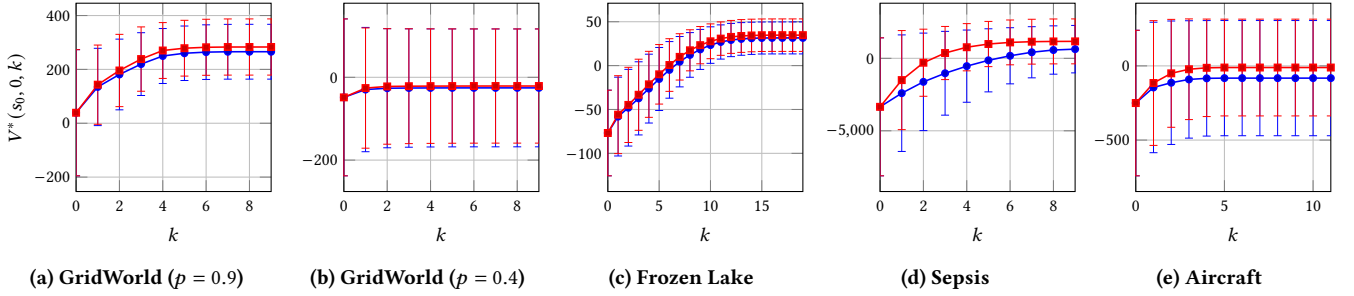
\begin{figure*}
\centering
\edef\cfimdpPanelWidth{\the\dimexpr0.145\textwidth\relax}
\edef\cfimdpPanelHeight{\the\dimexpr0.14\textwidth\relax}
\pgfplotsset{
    cfimdp panel/.style={
        scale only axis,
        width=\cfimdpPanelWidth,
        height=\cfimdpPanelHeight,
        xlabel={$k$},
        grid=major,
        font=\footnotesize,
        tick label style={font=\scriptsize},
        yticklabel style={xshift=1pt},
        ylabel shift=-3pt,
        every axis plot/.append style={
            mark size=1.2pt,
            line width=0.6pt,
            error bars/error bar style={line width=0.3pt},
            error bars/error mark options={rotate=90, mark size=1.5pt, line width=0.3pt},
        },
    }
}
\captionsetup[subfigure]{justification=centering}
\begin{subfigure}[t]{0.215\textwidth}
\centering
\begin{tikzpicture}
    \begin{axis}[
        cfimdp panel,
        xmin=0,
        xmax=9,
        ylabel={$V^*(s_0,0,k)$},
        legend columns=2,
        legend style={/tikz/every even column/.append style={column sep=1em}},
        legend to name=sharedLegend2
    ]

    \addplot+[
        color=blue,
        mark=*,
        error bars/.cd,
        y dir=both,
        y explicit
    ] coordinates {
        (0, 38.98) +- (0, 234.2844843347506)
        (1, 134.79974886000255) +- (1, 143.55234276756133)
        (2, 181.11859867678032) +- (2, 131.28665054556078)
        (3, 219.67185027332073) +- (3, 116.58860363531534)
        (4, 249.71742342845232) +- (4, 102.04492486244844)
        (5, 259.9355473812317) +- (5, 101.15582681476026)
        (6, 264.09758898032806) +- (6, 101.44779328314269)
        (7, 265.48604865044433) +- (7, 101.74884547199785)
        (8, 265.64428824616226) +- (8, 101.80632516127817)
        (9, 265.649894275875) +- (9, 101.80843346542855)
    };
    \addlegendentry{Data-driven CFIMDP}

    \addplot+[
        color=red,
        mark=square*,
        error bars/.cd,
        y dir=both,
        y explicit
    ] coordinates {
        (0, 38.98) +- (0, 234.2844843347506)
        (1, 143.22302626730746) +- (1, 147.1557995799284)
        (2, 195.91429274355193) +- (2, 134.2287720619955)
        (3, 238.3308317063109) +- (3, 119.48943884708012)
        (4, 270.14364955149443) +- (4, 103.68992000548751)
        (5, 278.72781906140017) +- (5, 103.59836326295454)
        (6, 282.13468027119467) +- (6, 104.06375390870947)
        (7, 283.0937905569028) +- (7, 104.30888521360238)
        (8, 283.1754435089571) +- (8, 104.34167903838107)
        (9, 283.17565926550293) +- (9, 104.34172515573152)
    };
    \addlegendentry{Exact CFIMDP}
    \end{axis}
    \end{tikzpicture}
\caption{{GridWorld ($p=0.9$)}}    \label{fig:gridworld}
\end{subfigure}%
\hfill
\begin{subfigure}[t]{0.19\textwidth}
\centering
\begin{tikzpicture}
    \begin{axis}[
        cfimdp panel,
        xmin=0,
        xmax=9
    ]

    \addplot+[
        color=blue,
        mark=*,
        error bars/.cd,
        y dir=both,
        y explicit
    ] coordinates {
        (0, -48.48) +- (0, 189.54590367507285)
        (1, -29.619009787879758) +- (1, 150.14791853839662)
        (2, -26.386447207908983) +- (2, 143.83889223063383)
        (3, -25.900523420065152) +- (3, 142.99005473448887)
        (4, -25.686107667822636) +- (4, 142.616931669738)
        (5, -25.601536005353864) +- (5, 142.46736707959641)
        (6, -25.570504987635942) +- (6, 142.40722357428726)
        (7, -25.55591294733943) +- (7, 142.37682660293837)
        (8, -25.55427293362152) +- (8, 142.3734711539469)
        (9, -25.55427293362152) +- (9, 142.3734711539469)
    };
    \addplot+[
        color=red,
        mark=square*,
        error bars/.cd,
        y dir=both,
        y explicit
    ] coordinates {
        (0, -48.48) +- (0, 189.54590367507285)
        (1, -26.01095575466666) +- (1, 145.23811381149417)
        (2, -22.069118804906665) +- (2, 139.5541278448571)
        (3, -21.501946209706656) +- (3, 138.7280012475704)
        (4, -21.28218729514666) +- (4, 138.36039879004176)
        (5, -21.160499871146662) +- (5, 138.13974391630052)
        (6, -21.091684220586657) +- (6, 138.00326531363356)
        (7, -21.06132992042665) +- (7, 137.9415129245016)
        (8, -21.059062685866657) +- (8, 137.9372996128318)
        (9, -21.059062685866657) +- (9, 137.9372996128318)
    };
    \end{axis}
    \end{tikzpicture}
\caption{{GridWorld ($p=0.4$)}}
\end{subfigure}%
\hfill
\begin{subfigure}[t]{0.19\textwidth}
\centering
 \begin{tikzpicture}
    \begin{axis}[
        cfimdp panel,
        xmin=0,
        xmax=19
    ]

    \addplot+[
        color=blue,
        mark=*,
        error bars/.cd,
        y dir=both,
        y explicit
    ] coordinates {
        (0, -76.81) +- (0, 48.82042502887496)
        (1, -58.09627271043695) +- (1, 44.89337821023748)
        (2, -47.99102023863091) +- (2, 43.703434156851436)
        (3, -37.329743104148704) +- (3, 41.68245391280376)
        (4, -26.221597365955077) +- (4, 39.22426270865087)
        (5, -15.255793997336939) +- (5, 35.91857891243606)
        (6, -4.879552660087285) +- (6, 32.26376291449058)
        (7, 4.453485085909165) +- (7, 28.691129093569085)
        (8, 12.166104170123504) +- (8, 25.44293058166517)
        (9, 18.452467292984) +- (9, 22.791716324277452)
        (10, 23.39495760427462) +- (10, 20.811340131421733)
        (11, 26.923498295517817) +- (11, 19.498329446552965)
        (12, 29.152661396864925) +- (12, 18.927370497906125)
        (13, 30.39538273559182) +- (13, 18.610373071840485)
        (14, 31.11948331564742) +- (14, 18.349539280768823)
        (15, 31.460880769391206) +- (15, 18.251860160104066)
        (16, 31.53593239419183) +- (16, 18.24366568669524)
        (17, 31.555566128670065) +- (17, 18.243531754527538)
        (18, 31.565158559868763) +- (18, 18.24296641798767)
        (19, 31.569807084633272) +- (19, 18.24223483304952)
    };
    \addplot+[
        color=red,
        mark=square*,
        error bars/.cd,
        y dir=both,
        y explicit
    ] coordinates {
        (0, -76.81) +- (0, 48.82042502887496)
        (1, -55.914832321556304) +- (1, 44.4001545080534)
        (2, -44.79562335448227) +- (2, 42.84947822428805)
        (3, -33.21789890987079) +- (3, 40.523356752545816)
        (4, -21.35012275834039) +- (4, 37.61673721416898)
        (5, -9.930276297724) +- (5, 33.916018112179344)
        (6, 0.6568369978173695) +- (6, 30.00506484740474)
        (7, 9.940008960010823) +- (7, 26.394572533370944)
        (8, 17.439580127053677) +- (8, 23.401410538595915)
        (9, 23.293700173809917) +- (9, 21.22971207374961)
        (10, 27.702302678082674) +- (10, 19.836963207729614)
        (11, 30.75708337549435) +- (11, 19.054455574782647)
        (12, 32.698114590850885) +- (12, 18.77497270055438)
        (13, 33.75303510323518) +- (13, 18.683364403347582)
        (14, 34.33265811045936) +- (14, 18.575807160160533)
        (15, 34.63645666593117) +- (15, 18.51248797509012)
        (16, 34.69703820074842) +- (16, 18.508052982535254)
        (17, 34.707200582561164) +- (17, 18.5093975378037)
        (18, 34.70766450288023) +- (18, 18.50951843931451)
        (19, 34.70920375230639) +- (19, 18.50928656577958)
    };
    \end{axis}
    \end{tikzpicture}
\caption{{Frozen Lake}}
    \label{fig:frozen_lake}
\end{subfigure}%
\hfill
\begin{subfigure}[t]{0.19\textwidth}
\centering
\begin{tikzpicture}
    \begin{axis}[
        cfimdp panel,
        xmin=0,
        xmax=9
    ]

    \addplot+[
        color=blue,
        mark=*,
        error bars/.cd,
        y dir=both,
        y explicit
    ] coordinates {
        (0, -3347.00) +- (0, 4760.54)
        (1, -2404.33) +- (1, 4020.78)
        (2, -1623.35) +- (2, 3363.72)
        (3, -1030.67) +- (3, 2890.76)
        (4, -539.54)  +- (4, 2490.94)
        (5, -138.34)  +- (5, 2176.52)
        (6, 175.48)   +- (6, 1936.54)
        (7, 418.28)   +- (7, 1764.60)
        (8, 573.50)   +- (8, 1667.14)
        (9, 632.67)   +- (9, 1637.06)
    };

    \addplot+[
        color=red,
        mark=square*,
        error bars/.cd,
        y dir=both,
        y explicit
    ] coordinates {
        (0, -3347.00) +- (0, 4760.54)
        (1, -1502.46) +- (1, 3417.75)
        (2, -303.57)  +- (2, 2314.55)
        (3, 366.02)   +- (3, 1833.55)
        (4, 765.85)   +- (4, 1628.70)
        (5, 986.70)   +- (5, 1559.71)
        (6, 1098.43)  +- (6, 1546.16)
        (7, 1147.70)  +- (7, 1547.35)
        (8, 1165.54)  +- (8, 1549.14)
        (9, 1169.69)  +- (9, 1549.81)
    };

    \end{axis}
    \end{tikzpicture}
    \caption{Sepsis}
    \label{fig:sepsis}
\end{subfigure}%
\hfill
\begin{subfigure}[t]{0.19\textwidth}
\centering
\begin{tikzpicture}
    \begin{axis}[
        cfimdp panel,
        xmin=0,
        xmax=11
    ]

    \addplot+[
        color=blue,
        mark=*,
        error bars/.cd,
        y dir=both,
        y explicit
    ] coordinates {
        (0, -250.5) +- (0, 490.5486214433795)
        (1, -145.4474148263778) +- (1, 439.6890925582446)
        (2, -113.48318507187962) +- (2, 416.0384826246025)
        (3, -91.31982581002114) +- (3, 395.565745019047)
        (4, -83.8790662592863) +- (4, 389.22697408426643)
        (5, -82.57824235959386) +- (5, 388.26117072701476)
        (6, -82.54193173283582) +- (6, 388.20100262162606)
        (7, -82.54193173283582) +- (7, 388.20100262162606)
        (8, -82.54193173283582) +- (8, 388.20100262162606)
        (9, -82.54193173283582) +- (9, 388.20100262162606)
        (10, -82.54193173283582) +- (10, 388.20100262162606)
        (11, -82.54193173283582) +- (11, 388.20100262162606)
    };
    \addplot+[
        color=red,
        mark=square*,
        error bars/.cd,
        y dir=both,
        y explicit
    ] coordinates {
        (0, -250.5) +- (0, 490.5486214433795)
        (1, -115.50765286978371) +- (1, 420.4324447538014)
        (2, -50.165096383711756) +- (2, 363.6471970942694)
        (3, -22.693806342597664) +- (3, 339.3073062670981)
        (4, -12.996514532991242) +- (4, 328.51079665920247)
        (5, -11.110271550502393) +- (5, 326.1372735771754)
        (6, -11.102212188824378) +- (6, 326.12754863322067)
        (7, -11.102212188824256) +- (7, 326.12754863322056)
        (8, -11.102212188824224) +- (8, 326.12754863322056)
        (9, -11.102212188824215) +- (9, 326.1275486332205)
        (10, -11.102212188824215) +- (10, 326.1275486332205)
        (11, -11.102212188824215) +- (11, 326.1275486332205)
    };

    \end{axis}
    \end{tikzpicture}
    \caption{{Aircraft}}
    \label{fig:aircraft}
\end{subfigure}
\caption{Average pessimistic counterfactual value $V^{*}(s_0,0,k)$ over $100$ randomly sampled observed trajectories, on the \textcolor{blue}{data-driven CFIMDP} $(N=100{,}000, \alpha=0.1)$ (\textcolor{blue}{$\bullet$}) vs. the \textcolor{red}{\textbf{exact CFIMDP}} (\textcolor{red}{$\blacksquare$}), for each $k$. The data-driven value is a certified lower bound, so the gap to the exact curve is the price of learning the model from data. Error bars show standard deviation.}%
\label{fig: performance}
\end{figure*}
\begin{table}
\centering
\caption{Average worst-case counterfactual value $V^{\bot,\pi^*_k}(s_0)$ on the exact CFIMDP of policies derived from data-driven CFIMDPs ($N=100{,}000$, $\alpha=0.1$) and of the exact CFIMDP's pessimistic optimal policy, over $100$ observed trajectories $\tau$ with $k=|\tau|-1$. Values are mean $\pm$ standard deviation.}
\label{tab: robustness}%
\small
\setlength{\tabcolsep}{4pt}
\begin{tabular}{l c c c}
\toprule
\textbf{Environment}
& \makecell{\textbf{Exact} \\ \textbf{CFIMDP}}
& \makecell{\textbf{Data-driven} \\ \textbf{CFIMDP} \\ \textbf{(Cons.\ LB)}}
& \makecell{\textbf{Data-driven} \\ \textbf{CFIMDP} \\ \textbf{(Gumbel-max)}} \\
\midrule
GridWorld ($p=0.9$) & $283 \pm 104$ & $283 \pm 104$ & $227 \pm 233$ \\
GridWorld ($p=0.4$) & $-21.1 \pm 138$ & $-21.9 \pm 139$ & $-56.6 \pm 165$ \\
Frozen Lake & $34.7 \pm 18.5$ & $34.6 \pm 18.5$ & $-4.80 \pm 49.5$ \\
Aircraft & $-11.1 \pm 326$ & $-24.6 \pm 344$ & $-78.3 \pm 408$ \\
\bottomrule
\end{tabular}
\end{table}

\subsection{Off-Policy Evaluation}
\label{sec: correctness}
We first check the correctness of our counterfactual inference method for IMDPs in an off-policy evaluation (OPE) setting \citep{oberst2019counterfactual,buesing2019woulda}, \updated{estimating the expected return of an optimal target policy $\pi^*$ using trajectories generated by a suboptimal behavioural policy $\pi$.}
\updated{We construct a data-driven IMDP given a fixed number $N$ of observed trajectories and confidence $\alpha=0.1$, sample $100$ trajectories from the true MDP under $\pi$, build a CFIMDP for each, and compute upper and lower bounds on the expected counterfactual return from the initial state. The counterfactual value averaged over observed trajectories should equal the interventional value. Hence, if the method is unbiased, the average bounds should contain the true interventional value $V^{\pi^*}(s_0)$. Table \ref{tab:cf_correctness_simulated_gridworld_0.9} shows the results for the GridWorld ($p=0.9$) MDP (similar results for GridWorld ($p=0.4$) can be found in Appendix \ref{app: additional experiments}), alongside the exact bounds from the true MDP following \citet{lally2025robustcounterfactualinferencemarkov}. The results confirm that our approach is valid and unbiased: the data-driven CFIMDP bounds consistently contain the true expected return, and converge towards the exact bounds as the number $N$ of trajectories increases.}
\begin{figure*}
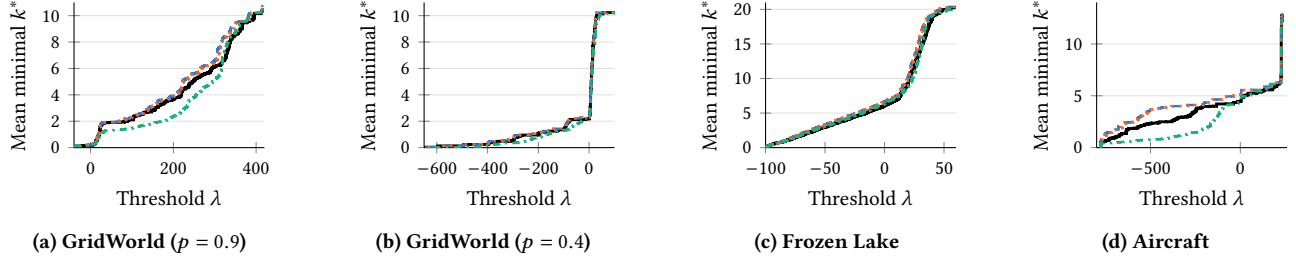

  \centering
  \pgfplotsset{optimalk axis/.style={legend to name=optimalklegend, legend columns=4}}
  \begin{subfigure}[b]{0.24\textwidth}
    \centering
    \input{graphs/gridworld_actual_causes}
    \caption{GridWorld ($p=0.9$)}
  \end{subfigure}\hfill
  \begin{subfigure}[b]{0.24\textwidth}
    \centering
    \input{graphs/gridworld_0.4_actual_causes}
    \caption{GridWorld ($p=0.4$)}
  \end{subfigure}\hfill
  \begin{subfigure}[b]{0.24\textwidth}
    \centering
    \pgfkeysifdefined{/pgfplots/optimalk axis/.@cmd}{}{\pgfplotsset{optimalk axis/.style={}}}%
\begin{tikzpicture}
\begin{axis}[
    width=0.96\linewidth,
    height=0.82\linewidth,
    xmin=-100, xmax=60,
    ymin=0, ymax=21,
    xlabel={Threshold $\lambda$},
    ylabel={Mean minimal $k^*$},
    axis lines=left,
    axis line style={-},
    ytick distance=5,
    ymajorgrids,
    grid style={line width=0.3pt, draw=black!12},
    tick label style={font=\footnotesize},
    label style={font=\small},
    legend style={
        font=\footnotesize,
        draw=none,
        at={(0.5,1.03)},
        anchor=south,
        legend columns=2,
        /tikz/every even column/.append style={column sep=0.8em},
    },
    legend cell align=left,
    every axis plot/.append style={const plot mark left, no markers},
    optimalk axis,
]
    \addplot[black, line width=1.1pt]             table[col sep=tab, x=threshold, y=exact_mean] {graphs/frozenlakedata.txt};
    \addlegendentry{Exact CFIMDP}
    \addplot[cfapprox, line width=0.9pt, dashed]  table[col sep=tab, x=threshold, y=approx_mean] {graphs/frozenlakedata.txt};
    \addlegendentry{Data-driven CFIMDP (Cons.\ LB)}
    \addplot[cftight, line width=1pt, dotted]     table[col sep=tab, x=threshold, y=tight_mean] {graphs/frozenlakedata.txt};
    \addlegendentry{Data-driven CFIMDP (Tight LB)}
    \addplot[cfgumbel, line width=0.9pt, dash dot] table[col sep=tab, x=threshold, y=gumbel_mean] {graphs/frozenlakedata.txt};
    \addlegendentry{Data-driven CFIMDP (Gumbel-max)}
\end{axis}
\end{tikzpicture}%
    \caption{Frozen Lake}
  \end{subfigure}
  \begin{subfigure}[b]{0.24\textwidth}
    \centering
    \pgfkeysifdefined{/pgfplots/optimalk axis/.@cmd}{}{\pgfplotsset{optimalk axis/.style={}}}%
\begin{tikzpicture}
\begin{axis}[
    width=0.96\linewidth,
    height=0.82\linewidth,
    xmin=-800, xmax=260,
    ymin=0, ymax=14,
    xlabel={Threshold $\lambda$},
    ylabel={Mean minimal $k^*$},
    axis lines=left,
    axis line style={-},
    ytick distance=5,
    ymajorgrids,
    grid style={line width=0.3pt, draw=black!12},
    tick label style={font=\footnotesize},
    label style={font=\small},
    legend style={
        font=\footnotesize,
        draw=none,
        at={(0.5,1.03)},
        anchor=south,
        legend columns=2,
        /tikz/every even column/.append style={column sep=0.8em},
    },
    legend cell align=left,
    every axis plot/.append style={const plot mark left, no markers},
    unbounded coords=discard,
    optimalk axis,
]
    \addplot[black, line width=1.1pt]             table[col sep=tab, x=threshold, y=exact_mean] {graphs/aircraftdata.txt};
    \addplot[cfapprox, line width=0.9pt, dashed]  table[col sep=tab, x=threshold, y=approx_mean] {graphs/aircraftdata.txt};
    \addplot[cftight, line width=1pt, dotted]     table[col sep=tab, x=threshold, y=tight_mean] {graphs/aircraftdata.txt};
    \addplot[cfgumbel, line width=0.9pt, dash dot] table[col sep=tab, x=threshold, y=gumbel_mean] {graphs/aircraftdata.txt};
\end{axis}
\end{tikzpicture}%
    \caption{Aircraft}
  \end{subfigure}
  \medskip
  \ref{optimalklegend}

  \caption{Minimum number of action changes $k^*$ needed for $V^{*}(s_0,0,k)$ to reach a threshold $\lambda$, averaged over $100$ randomly sampled observed trajectories. We compare the exact CFIMDP (black, solid) with data-driven CFIMDPs using conservative lower bounds (\textcolor{cfapprox}{dashed}), tight lower bounds (\textcolor{cftight}{dotted}), and the Gumbel-max model (\textcolor{cfgumbel}{dash-dot}).\protect\hyperref[fn:sepsis]{\textsuperscript{\ref*{fn:sepsis}}}}
  \label{fig:optimal_k}
\end{figure*}
\begin{figure*}[t]
  \centering
  \begin{subfigure}{\textwidth}
    \centering
    \begin{subfigure}[b]{0.19\textwidth}
      \centering
      \includegraphics[height=2cm, width=\textwidth, keepaspectratio]{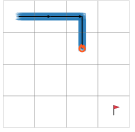}
      \caption*{(i) $k=0$}
    \end{subfigure}\hfill
    \begin{subfigure}[b]{0.19\textwidth}
      \centering
      \includegraphics[height=2cm, width=\textwidth, keepaspectratio]{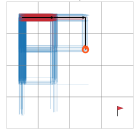}
      \caption*{(ii) $k=1$}
    \end{subfigure}\hfill
    \begin{subfigure}[b]{0.19\textwidth}
      \centering
      \includegraphics[height=2cm, width=\textwidth, keepaspectratio]{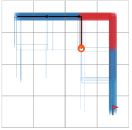}
      \caption*{(iii) $k=2$}
    \end{subfigure}\hfill
    \begin{subfigure}[b]{0.19\textwidth}
      \centering
      \includegraphics[height=2cm, width=\textwidth, keepaspectratio]{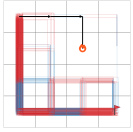}
      \caption*{(iv) $k=5$}
    \end{subfigure}\hfill
    \begin{subfigure}[b]{0.19\textwidth}
      \centering
      \includegraphics[height=2cm, width=\textwidth, keepaspectratio]{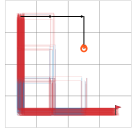}
      \caption*{(v) $k=8$}
    \end{subfigure}
    \caption{GridWorld: distribution of counterfactual trajectories. The observed path is shown in black, and counterfactual paths in blue where the observed action is taken and in red where the action has been changed.}
  \end{subfigure}
  \medskip
  \begin{subfigure}{\textwidth}
    \centering
    \begin{subfigure}[b]{0.211\textwidth}
      \centering
      \includegraphics[height=3.4cm, width=\textwidth, keepaspectratio]{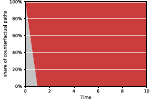}
      \caption*{(i) $k=0$}
    \end{subfigure}\hfill
    \begin{subfigure}[b]{0.18\textwidth}
      \centering
      \includegraphics[height=2.8cm, width=\textwidth, keepaspectratio]{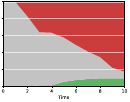}
      \caption*{(ii) $k=1$}
    \end{subfigure}\hfill
    \begin{subfigure}[b]{0.18\textwidth}
      \centering
      \includegraphics[height=2.8cm, width=\textwidth, keepaspectratio]{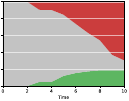}
      \caption*{(iii) $k=2$}
    \end{subfigure}\hfill
    \begin{subfigure}[b]{0.18\textwidth}
      \centering
      \includegraphics[height=2.8cm, width=\textwidth, keepaspectratio]{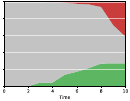}
      \caption*{(iv) $k=5$}
    \end{subfigure}\hfill
    \begin{subfigure}[b]{0.18\textwidth}
      \centering
      \includegraphics[height=2.8cm, width=\textwidth, keepaspectratio]{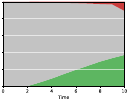}
      \caption*{(v) $k=8$}
    \end{subfigure}
    \caption{Sepsis: share of counterfactual paths at each time step where the patient has recovered (green), remains in treatment (grey), or has died (red).}
  \end{subfigure}
\caption{Behaviour of the data-driven counterfactual policy for a single observed catastrophic path, with varying $k$, under a CFMDP adversarially chosen from the exact CFIMDP.}
\label{fig: counterfactual visualisation}
\end{figure*}
\subsection{Evaluating Performance}
\label{sec: evaluation performance}
\updated{Next, we evaluate the impact of learning the model from data by comparing the performance of the counterfactual policies derived from our data-driven IMDP (with conservative lower bounds) against those derived from the exact CFIMDP. We use a suboptimal observational policy $\pi$ that selects a random action over the optimal action with probability $0.5$, construct an IMDP with hyperparameters $(N = 100,000, \alpha=0.1)$, and derive CFIMDPs for $100$ paths sampled under $\pi$ using both approaches. We then compare the average values $V^{*}(s_0,0,k)$ of the resulting pessimistic optimal policies, where $s_0$ is the MDP's initial state, for each maximum number of action changes $k$.} 

\updated{For GridWorld and Frozen Lake (Figures \ref{fig:gridworld}-\ref{fig:frozen_lake}), our approach yields estimates of $V^{*}(s_0,0,k)$ close to those obtained from the exact CFIMDP, showing that the statistical uncertainty in the MDP has little impact on the robust counterfactual policies. For Sepsis and Aircraft (Figures \ref{fig:sepsis} and \ref{fig:aircraft}), $V^{*}(s_0,0,k)$ is slightly lower, as the larger state spaces require wider bounds in the IMDP to maintain an overall significance level of $\alpha=0.1$; with more data, these estimates would tighten. The resulting policies may therefore identify more conservative \revised{interventions}, i.e., change more actions than strictly necessary. However, since the statistical guarantees on the IMDP carry over to the CFIMDP, these \revised{interventions} remain valid at the stated confidence level. Finally, as shown in Section \ref{sec: evaluation robustness}, policies derived from data-driven IMDPs tend to exceed their estimated performance when evaluated on the true CFIMDP, suggesting that these policies are generally accurate and close to the optimal policy obtained without statistical uncertainty.}

\subsection{Evaluating Robustness}
\label{sec: evaluation robustness}
\updated{To evaluate robustness, we compare the counterfactual policies generated by our approach (with conservative lower bounds) and from the data-driven Gumbel-max approach on the exact CFIMDP, i.e., by their worst-case counterfactual value across all causal models compatible with the true MDP (Table \ref{tab: robustness})\footnote{\label{fn:sepsis}We omit Sepsis, as generating its Gumbel-max CFIMDPs is computationally prohibitive.}. Our policies perform close to the optimal policies learned from the true CFIMDP, showing little impact from statistical uncertainty in these examples. In contrast, the Gumbel-max CFIMDP policies perform suboptimally, suggesting that committing to a specific causal model (even when accounting for uncertainty in the transition dynamics) can be unreliable if the assumed model is misspecified.}

\subsection{Computing Robust Interventions}
\label{sec: actual causes}
\updated{We now evaluate how well our approach identifies \revised{minimal action changes that achieve a counterfactual value $\geq \lambda$. To do so, we use} the same $100$ observed trajectories \revised{ as before}. For each trajectory and $k \in \{0,\dots,|\tau|-1\}$, we compute the pessimistic policy $\pi^{*}_k$ and its value $V^{*}(s_0,0,k)$ under the exact CFIMDP, our data-driven CFIMDPs (conservative and tight lower bounds), and the Gumbel-max CFIMDP, and report the smallest $k^*$ such that $V^{*}(s_0,0,k) \ge \lambda$. Figure~\ref{fig:optimal_k} shows that both data-driven approaches identify 
\revised{interventions involving similar numbers of action changes to} the exact CFIMDP, being only slightly more conservative, whereas Gumbel-max often underestimates $k^*$, particularly in GridWorld $(p=0.9)$ and Aircraft, so its \revised{interventions} may be unreliable.} \updated{Figure~\ref{fig: counterfactual visualisation} visualises the identified \revised{interventions} for one catastrophic trajectory. In GridWorld, with $k=1$ the policy changes the second action to move back towards the start, so most counterfactual paths avoid the fire, but none reach the goal; with $k=2$, it instead changes the action that led into the fire, allowing over 60\% of counterfactual paths to reach the goal. The 
\revised{size of a robust intervention} 
thus depends on the outcome compared against: avoiding catastrophe requires one change, while reaching the goal requires two. Similarly, in Sepsis, mortality falls, and recovery rises as $k$ grows. Appendix~\ref{app: additional experiments} shows these counterfactual paths in more detail, plotting each vital sign and treatment over time for every $k$.}

\subsection{Execution Times and Bound Widths}
\updated{Finally, we assess the impact of using conservative lower bounds (closed-form, efficient) instead of tight lower bounds (optimisation-based, expensive). For the same $100$ paths, we derive the CFIMDPs with both approaches and measure the percentage of transitions whose lower bound changes when the tight bound is used, both across all transitions and across non-zero probability transitions only, together with the average change in bound width. The full results are reported in Table~\ref{tab:combined_results} in Appendix \ref{app: additional experiments}. The overall effect is small: at most $4.4\%$ of all transitions change, and among non-zero probability transitions, at most $21\%$ in GridWorld ($p=0.9$) and at most $6\%$ elsewhere, with an average width change of at most $0.06$ for the transitions that change. As the previous experiments show, this has little impact on the derived policies, and since tight lower bounds require solving an additional optimisation problem and are slower to compute, the conservative approach is preferable in practice. While comparable in the smaller environments, the Gumbel-max approach scales poorly with the MDP state space, taking $293$s on average for Aircraft compared with $10$s for our conservative bounds, and is computationally prohibitive for Sepsis.}

\section{Conclusion}
We have presented a counterfactual inference approach for MDPs that is robust to both uncertainty in the underlying transition dynamics and the unknown causal model. We derived closed-form bounds on counterfactual transition probabilities, enabling scalable inference in larger MDPs. 
\revised{Building on these bounds, we introduced minimal robust counterfactual interventions: non-stationary policies that guarantee a desired expected return under all admissible counterfactual dynamics, while minimising the maximum number of action changes. 
Experiments across diverse MDP environments demonstrate the robustness and scalability of our approach, with counterfactual interventions closely tracking those of the exact counterfactual MDP.}  
Future work will address unobserved confounding in offline trajectory datasets 
and extend our approach to continuous and partially-observable environments.

\newpage
\balance
\bibliographystyle{ACM-Reference-Format} 
\bibliography{example_paper}

\onecolumn
\begin{center}
{\Huge\bfseries Robust Counterfactuals for Markov Decision Processes
under Statistical and Causal Uncertainties (Supplementary Material)\par}
\end{center}
\vspace{1.5em}
\appendix
\section{Counterfactual Probability Bounds for MDPs}
\label{app: assumptions and bounds}

We copy the counterfactual probability bounds proven by \citet{lally2025robustcounterfactualinferencemarkov} below, under the assumptions of counterfactual stability and counterfactual monotonicity:

\begin{definition}[Counterfactual stability]\label{def: counterfactual stability} An MDP SCM \eqref{eq:mdp_scm} satisfies \textit{counterfactual stability} if, given we observed the transition $s_t, a_t \rightarrow s_{t+1}$ at time $t$, the counterfactual outcome under a different state-action pair $(s, a)$ will not change to some $S_{t+1} = s'\neq s_{t+1}$ unless $\dfrac{P(s_{t+1} \mid s, a)}{P(s_{t+1} \mid s_t, a_t)} < \dfrac{P(s' \mid s, a)}{P(s' \mid s_t, a_t)}$.
\end{definition}

\begin{definition}[Counterfactual monotonicity]
    An MDP SCM \eqref{eq:mdp_scm} satisfies \textit{counterfactual monotonicity} if, upon observing the transition $s_t, a_t \rightarrow s_{t+1}$, then, $\forall s \in \mathcal{S}, a \in \mathcal{A}$: 
    \begin{enumerate}[label=(Mon\arabic*)]
        \item $\tilde{P}_t(s_{t+1} \mid s, a) \geq P(s_{t+1} \mid s, a)$ (i.e., observing an outcome cannot make it less likely in the counterfactual world), and
        \item $\forall s'\neq s_{t+1}$ with $P(s' \mid s_t, a_t)>0$, $\tilde{P}_t(s' \mid s, a) \leq P(s' \mid s, a)$ (i.e., not observing a possible outcome cannot make it more likely in the counterfactual world).
    \end{enumerate}
\end{definition}

For any MDP $\mathcal{M}$ and observed transition $s_t, a_t \rightarrow s_{t+1}$, we can bound the counterfactual probability of any transition $s, a \rightarrow s'$ as follows.

\begin{theorem}
\label{theorem: observed state appendix}
When the counterfactual state-action pair is the observed state-action pair, $(s, a) = (s_{t}, a_t)$, the counterfactual probability bounds are as follows:
{
\[\tilde{P}_{t}^{\it LB}(s_{t+1} \mid s_t, a_t) = \tilde{P}_{t}^{\it UB}(s_{t+1} \mid s_t, a_t) = 1\] 
\[\forall s' \in \mathcal{S} \setminus \{s_{t+1}\}, \tilde{P}_{t}^{\it LB}(s' \mid s_t, a_t) = \tilde{P}_{t}^{\it UB}(s' \mid s_t, a_t) = 0\]
}
\end{theorem}

\begin{theorem}
\label{theorem: ub disjoint}
When the counterfactual state-action pair $(s, a)\neq (s_t, a_t)$ and has disjoint support from the observed $(s_t, a_t)$, the counterfactual probability bounds are as follows, $\forall s' \in \mathcal{S}$:
\[
\tilde{P}_{t}^{\it UB}(s' \mid s, a) =  
\begin{cases}
\dfrac{P(s' \mid s, a)}{P(s_{t+1} \mid s_t, a_t)} & \textnormal{ if }P(s' \mid s, a) < P(s_{t+1} \mid s_t, a_t) \\
1 &\textnormal{ otherwise}
\end{cases}
\]

\[
\tilde{P}_{t}^{\it LB}(s' \mid s, a) =  
\begin{cases}
\dfrac{%
  P(s' \mid s, a) - \bigl(1 - P(s_{t+1} \mid s_t, a_t)\bigr)%
}{%
  P(s_{t+1} \mid s_t, a_t)%
} 
& 
\begin{aligned}
&\textnormal{if } P(s' \mid s, a) > \\&1 - P(s_{t+1} \mid s_t, a_t)
\end{aligned} \\[1.5mm]
0 & \textnormal{otherwise.}
\end{cases}
\]
\end{theorem}

\begin{theorem}
\label{theorem:ub overlapping}
When the counterfactual state-action pair $(s, a)\neq (s_t, a_t)$ and has overlapping support with the observed $(s_t, a_t)$, the counterfactual probability bounds are as follows, $\forall s' \in \mathcal{S}$:

\[
\tilde{P}_{t}^{\it UB}(s' \mid s, a) = \begin{cases}
\dfrac{\min\left(\begin{array}{@{}c@{}} P(s_{t+1} \mid s_t, a_t), P(s_{t+1} \mid s, a) \end{array}\right)}{P(s_{t+1} \mid s_t, a_t)} & \text{ if } s' = s_{t+1} \\
0 & \textnormal{ if CS conditions hold} \\
\min\left(\begin{array}{@{}c@{}} P(s' \mid s, a), 1 - P(s_{t+1} \mid s, a) \end{array}\right) & \textnormal{ if } P({s' \mid s_t, a_t}) > 0 \\
\min\left(\begin{array}{@{}c@{}} 1 - P(s_{t+1} \mid s, a), \dfrac{P(s' \mid s, a)}{P(s_{t+1} \mid s_t, a_t)} \end{array}\right) & \textnormal{ otherwise}
\end{cases}
\]
{
\[
\tilde{P}_{t}^{\it LB}(s' \mid s, a) = \begin{cases}
\max\left(P(s' \mid s, a), 1 - \sum_{s'' \in \mathcal{S}\setminus\{s'\}}{\tilde{P}_{t}^{\it UB}(s'' \mid s, a)}\right) & \textnormal{ if $s'= s_{t+1}$}\\
0 & \textnormal{ if CS conditions hold}\\
\max\left(0, 1 - \sum_{s'' \in \mathcal{S}\setminus\{s'\}}{\tilde{P}_{t}^{\it UB}(s'' \mid s, a)}\right) & \textnormal{ otherwise}
\end{cases}
\]
}
\noindent where the counterfactual stability (CS) conditions are $P({s' \mid s_t, a_t}) > 0$  and  $\dfrac{P(s_{t+1} \mid s, a)}{P(s_{t+1} \mid s_t, a_t)} \geq \dfrac{P(s' \mid s, a)}{P({s' \mid s_t, a_t})}$
\end{theorem}

\section{Proofs}
\label{app: proofs}
Let $\mathcal{M}_\updownarrow$ be an arbitrary IMDP. For a given counterfactual transition $s, a \rightarrow s'$ and observed transition $s_t, a_t \rightarrow s_{t+1}$, the optimisation problems defined in \eqref{eq: optimisation ub} and \eqref{eq: optimisation lb} yield the upper and lower bounds on the counterfactual probability of $s, a \rightarrow s'$ respectively, under the following assumptions.

\begin{assumption}[Normalisation]
    For any IMDP $\mathcal{M}_\updownarrow$, we assume the transition bounds have been normalised in the sense that, $\forall s, a, s'$:
    \[P^{\it LB}_{\mathcal{M}_\updownarrow}(s' \mid s, a) \geq 1 - \sum_{s'' \in \mathcal{S}\setminus\{s'\}}P^{\it UB}_{\mathcal{M}_\updownarrow}(s'' \mid s, a)\]
    \[P^{\it UB}_{\mathcal{M}_\updownarrow}(s' \mid s, a) \leq 1 - \sum_{s'' \in \mathcal{S}\setminus\{s'\}}P^{\it LB}_{\mathcal{M}_\updownarrow}(s'' \mid s, a)\]
\label{ass: normalisation}
\end{assumption}

This is not a strong assumption, because any feasible MDP $\mathcal{M} \in \mathcal{M}_\updownarrow$ must have transition probabilities that sum to $1$, so any IMDP can be tightened to meet the above conditions without altering the set of compatible MDPs.

\begin{assumption}[Non-zero Lower Bounds of Observed Transitions]
    Given any arbitrary IMDP $\mathcal{M}_\updownarrow$, we assume that, for any feasible $\mathcal{M} \in \mathcal{M}_\updownarrow$, $P_{\mathcal{M}}(s_{t+1} \mid s_t, a_t) > 0$, i.e., we assume $P^{LB}_{\mathcal{M}_\updownarrow}(s_{t+1} \mid s_t, a_t) > 0$.
\label{ass: observed transition}
\end{assumption}

This assumption must hold, because we observed the transition $s_t, a_t \rightarrow s_{t+1}$, hence the true MDP must have $P_{\mathcal{M}}(s_{t+1} \mid s_t, a_t) > 0$. 

\subsection{Observed $(s_t, a_t)$}
\begin{theorem}
\label{theorem: observed imdp}
Assume $(s, a) = (s_t, a_t)$. For outgoing transitions from the observed state-action pair $(s_t, a_t)$, the optimisation problem in \eqref{eq: optimisation ub} will produce an upper bound of:
\[
\tilde{P}^{UB}_{\mathcal{M}_\updownarrow}(s' \mid s, a) =
\begin{cases}
    1 &\text{if $s' = s_{t+1}$}\\
    0 &\text{otherwise}
\end{cases}
\]
and the optimisation problem in \eqref{eq: optimisation lb} will produce a lower bound of:
\[
\tilde{P}^{LB}_{\mathcal{M}_\updownarrow}(s' \mid s, a) =
\begin{cases}
    1 &\text{if $s' = s_{t+1}$}\\
    0 &\text{otherwise}
\end{cases}
\]
across all $\mathcal{M} \in \mathcal{\mathcal{M}_\updownarrow}$.
\end{theorem}
\begin{proof}
For all MDPs $\mathcal{M} \in \mathcal{M}_\updownarrow$, we know that the counterfactual probability for all transitions from the observed state-action pair $(s_t, a_t)$ will be:
\[\tilde{P}_{\mathcal{M},t}^{\it LB}(s_{t+1} \mid s_t, a_t) = \tilde{P}_{\mathcal{M},t}^{\it UB}(s_{t+1} \mid s_t, a_t) = 1\] 
\[\forall s' \in \mathcal{S}\setminus \{s_{t+1}\}, \tilde{P}_{\mathcal{M},t}^{\it LB}(s' \mid s_t, a_t) = \tilde{P}_{\mathcal{M},t}^{\it UB}(s' \mid s_t, a_t) = 0\]
Therefore, if $s' = s_{t+1}$:

\[
\tilde{P}^{UB}_{\mathcal{M}_\updownarrow,t}(s' \mid s, a) = {\max}_{\mathcal{M} \in \mathcal{M_\updownarrow}} \left( \tilde{P}^{UB}_{\mathcal{M},t}(s_{t+1}\mid s_t, a_t) \right) = {\max}_{\mathcal{M} \in \mathcal{M_\updownarrow}}1 = 1
\]
\[
\tilde{P}^{LB}_{\mathcal{M}_\updownarrow,t}(s' \mid s, a) = {\min}_{\mathcal{M} \in \mathcal{M_\updownarrow}} \left( \tilde{P}^{LB}_{\mathcal{M},t}(s_{t+1}\mid s_t, a_t) \right) = {\min}_{\mathcal{M} \in \mathcal{M_\updownarrow}}1 = 1
\]
and if $s' \in \mathcal{S}\setminus\{s_{t+1}\}$:
\[
\tilde{P}^{UB}_{\mathcal{M}_\updownarrow,t}(s' \mid s, a) = {\max}_{\mathcal{M} \in \mathcal{M_\updownarrow}} \left( \tilde{P}^{UB}_{\mathcal{M},t}(s' \mid s_t, a_t) \right) = {\max}_{\mathcal{M} \in \mathcal{M_\updownarrow}}0 = 0
\]
\[
\tilde{P}^{LB}_{\mathcal{M}_\updownarrow,t}(s' \mid s, a) = {\min}_{\mathcal{M} \in \mathcal{M_\updownarrow}} \left( \tilde{P}^{LB}_{\mathcal{M},t}(s' \mid s_t, a_t) \right) = {\min}_{\mathcal{M} \in \mathcal{M_\updownarrow}}0 = 0
\]
\end{proof}
\subsection{Disjoint $(s, a)$}
\subsubsection{Upper Bound}
\begin{theorem}
\label{theorem: upper disjoint imdp}
Assume $(s, a) \neq (s_t, a_t)$, and assume that the support of $(s, a)$ is disjoint from the support of $(s_t, a_t)$ in the IMDP $\mathcal{M}_\updownarrow$. In this case, the optimisation problem in \eqref{eq: optimisation ub} will produce an upper bound of:
\[
\tilde{P}_{\mathcal{M}_\updownarrow, t}^{\it UB}(s' \mid s, a) = 
\begin{cases}
\dfrac{P^{\it UB}_{\mathcal{M}_\updownarrow}(s' \mid s, a)}{P^{\it LB}_{\mathcal{M}_\updownarrow}(s_{t+1} \mid s_t, a_t)} & \text{if $P^{UB}_{\mathcal{M}_\updownarrow}(s' \mid s, a) < P^{LB}_{\mathcal{M}_\updownarrow}(s_{t+1} \mid s_t, a_t)$} \\
1 & \text{otherwise} \\
\end{cases}
\]
\end{theorem}
\begin{proof}
Because the supports of $(s, a)$ and $(s_t, a_t)$ are disjoint in the IMDP $\mathcal{M}_\updownarrow$, this means that $\forall s' \in \mathcal{S}, P^{\it UB}_{\mathcal{M}_\updownarrow}(s' \mid s, a) > 0 \Rightarrow P^{\it UB}_{\mathcal{M}_\updownarrow}(s' \mid s_t, a_t) = 0$, and $\forall s' \in \mathcal{S}, P^{\it UB}_{\mathcal{M}_\updownarrow}(s' \mid s_t, a_t) > 0 \Rightarrow P^{\it UB}_{\mathcal{M}_\updownarrow}(s' \mid s, a) = 0$. Therefore, in any MDP $\mathcal{M} \in \mathcal{M_{\updownarrow}}$, the supports of $(s, a)$ and $(s_t, a_t)$ are guaranteed to be disjoint as well. 

For any MDP $\mathcal{M} \in \mathcal{M_{\updownarrow}}$, the proven closed-form solution to the upper bound $\tilde{P}_{\mathcal{M}, t}^{\it UB}(s' \mid s, a)$ is:

\[\tilde{P}_{\mathcal{M}, t}^{\it UB}(s' \mid s, a) = 
\begin{cases}
\dfrac{P_{\mathcal{M}}(s' \mid s, a)}{P_{\mathcal{M}}(s_{t+1} \mid s_t, a_t)} & \text{if $P_\mathcal{M}(s' \mid s, a) < P_\mathcal{M}(s_{t+1} \mid s_t, a_t)$} \\
1 & \text{otherwise} \\
\end{cases}
\]

Therefore, to identify the upper bound, we need to maximise this over all MDPs $\mathcal{M} \in \mathcal{M}_\updownarrow$. Clearly, if there exist MDPs $\mathcal{M} \in \mathcal{M}_\updownarrow$ where $P_\mathcal{M}(s' \mid s, a) < P_\mathcal{M}(s_{t+1} \mid s_t, a_t)$, then $\dfrac{P_\mathcal{M}(s' \mid s, a)}{P_\mathcal{M}(s_{t+1} \mid s_t, a_t)} < 1$, so the absolute maximum value of $\tilde{P}_{\mathcal{M}, t}^{\it UB}(s' \mid s, a)$ across all MDPs cannot be greater than $1$.

Consider whether $P^{UB}_{\mathcal{M}_\updownarrow}(s' \mid s, a) < P^{LB}_{\mathcal{M}_\updownarrow}(s_{t+1} \mid s_t, a_t)$:

\begin{enumerate}
    \item If $P^{UB}_{\mathcal{M}_\updownarrow}(s' \mid s, a) \geq P^{LB}_{\mathcal{M}_\updownarrow}(s_{t+1} \mid s_t, a_t)$, then, because of Assumption \ref{ass: normalisation}, there must exist an MDP $\mathcal{M} \in \mathcal{M}_\updownarrow$ where $P_\mathcal{M}(s_{t+1} \mid s_t, a_t) = P^{\it LB}_{\mathcal{M}_\updownarrow}(s_{t+1} \mid s_t, a_t)$ and $P_\mathcal{M}(s' \mid s, a) = P^{\it UB}_{\mathcal{M}_\updownarrow}(s' \mid s, a)$. Because $P^{UB}_{\mathcal{M}_\updownarrow}(s' \mid s, a) \geq P^{LB}_{\mathcal{M}_\updownarrow}(s_{t+1} \mid s_t, a_t)$, we know that in MDP $\mathcal{M}$, the upper bound counterfactual probability $\tilde{P}_{\mathcal{M}, t}^{\it UB}(s' \mid s, a) = 1$. 
    Since this is the maximum possible counterfactual probability, this means that $\tilde{P}_{\mathcal{M}_\updownarrow, t}^{\it UB}(s' \mid s, a) = 1$.
    
    \item Otherwise, if $P^{UB}_{\mathcal{M}_\updownarrow}(s' \mid s, a) < P^{LB}_{\mathcal{M}_\updownarrow}(s_{t+1} \mid s_t, a_t)$, then, for every MDP $\mathcal{M} \in \mathcal{M}_\updownarrow$, $P_{\mathcal{M}}(s' \mid s, a) < P_{\mathcal{M}}(s_{t+1} \mid s_t, a_t)$. Therefore, for every MDP, the upper bound counterfactual probability for the transition will be $\tilde{P}_{\mathcal{M}, t}^{\it UB}(s' \mid s, a) = 
    \dfrac{P_\mathcal{M}(s' \mid s, a)}{P_\mathcal{M}(s_{t+1} \mid s_t, a_t)}$. Therefore, the maximum upper bound counterfactual probability is:

    \begin{align*}
    \tilde{P}^{UB}_{\mathcal{M}_\updownarrow}(s' \mid s, a) &= {\max}_{\mathcal{M} \in \mathcal{M_\updownarrow}} \left( \tilde{P}^{UB}_{\mathcal{M}, t}(s' \mid s, a) \right) \\
    &= {\max}_{\mathcal{M} \in \mathcal{M_\updownarrow}} \left( \dfrac{P_{\mathcal{M}}(s' \mid s, a)}{P_{\mathcal{M}}(s_{t+1} \mid s_t, a_t)}\right) \\
    &= \dfrac{P^{\it UB}_{\mathcal{M}_\updownarrow}(s' \mid s, a)}{P^{\it LB}_{\mathcal{M}_\updownarrow}(s_{t+1} \mid s_t, a_t)}
    \end{align*}

    Because of Assumption \ref{ass: normalisation}, there must exist an MDP $\mathcal{M}$ where $P_\mathcal{M}(s_{t+1} \mid s_t, a_t) = P^{\it LB}_{\mathcal{M}_\updownarrow}(s_{t+1} \mid s_t, a_t)$ and $P_\mathcal{M}(s' \mid s, a) = P^{\it UB}_{\mathcal{M}_\updownarrow}(s' \mid s, a)$. Also, note that because of Assumption \ref{ass: observed transition}, $P^{\it LB}_{\mathcal{M}_\updownarrow}(s_{t+1} \mid s_t, a_t) > 0$.
    
    Therefore, the maximum upper bound counterfactual probability in this case is 
    \[
    \tilde{P}^{UB}_{\mathcal{M}_\updownarrow}(s' \mid s, a) = \dfrac{P^{\it UB}_{\mathcal{M}_\updownarrow}(s' \mid s, a)}{P^{\it LB}_{\mathcal{M}_\updownarrow}(s_{t+1} \mid s_t, a_t)}
    \]

    Also, note that because of Assumption \ref{ass: observed transition}, $P^{\it LB}_{\mathcal{M}_\updownarrow}(s_{t+1} \mid s_t, a_t) > 0$.
\end{enumerate}
We can combine these cases as follows:
\[
\tilde{P}_{\mathcal{M}_\updownarrow, t}^{\it UB}(s' \mid s, a) = 
\begin{cases}
\dfrac{P^{\it UB}_{\mathcal{M}_\updownarrow}(s' \mid s, a)}{P^{\it LB}_{\mathcal{M}_\updownarrow}(s_{t+1} \mid s_t, a_t)} & \text{if $P^{UB}_{\mathcal{M}_\updownarrow}(s' \mid s, a) < P^{LB}_{\mathcal{M}_\updownarrow}(s_{t+1} \mid s_t, a_t)$} \\
1 & \text{otherwise} \\
\end{cases}
\]
\end{proof}

\subsubsection{Lower Bound}
\begin{theorem}
\label{theorem: lower disjoint imdp}
Assume $(s, a) \neq (s_t, a_t)$, and assume that the support of $(s, a)$ is disjoint from the support of $(s_t, a_t)$ in the IMDP $\mathcal{M}_\updownarrow$. In this case, the optimisation problem in \eqref{eq: optimisation lb} will produce a lower bound of:
\[\tilde{P}_{\mathcal{M}_\updownarrow, t}^{\it LB}(s' \mid s, a) = 
\begin{cases}
\dfrac{P^{\it LB}_{\mathcal{M}_\updownarrow}(s' \mid s, a) - (1 - P^{\it LB}_{\mathcal{M}_\updownarrow}(s_{t+1} \mid s_t, a_t))}{P^{\it LB}_{\mathcal{M}_\updownarrow}(s_{t+1} \mid s_t, a_t)} & \text{if $P^{\it LB}_{\mathcal{M}_\updownarrow}(s' \mid s, a) > 1 - P^{\it LB}_{\mathcal{M}_\updownarrow}(s_{t+1} \mid s_t, a_t)$}\\
0 & \text{otherwise}\\
\end{cases}
\]
\end{theorem}
\begin{proof}
Because the supports of $(s, a)$ and $(s_t, a_t)$ are disjoint in the IMDP $\mathcal{M}_\updownarrow$, this means that  $\forall s' \in \mathcal{S}, P^{\it UB}_{\mathcal{M}_\updownarrow}(s' \mid s, a) > 0 \Rightarrow P^{\it UB}_{\mathcal{M}_\updownarrow}(s' \mid s_t, a_t) = 0$, and $\forall s' \in \mathcal{S}, P^{\it UB}_{\mathcal{M}_\updownarrow}(s' \mid s_t, a_t) > 0 \Rightarrow P^{\it UB}_{\mathcal{M}_\updownarrow}(s' \mid s, a) = 0$. Therefore, in any MDP $\mathcal{M} \in \mathcal{M_{\updownarrow}}$, the supports of $(s, a)$ and $(s_t, a_t)$ are guaranteed to be disjoint as well. 

For any MDP $\mathcal{M} \in \mathcal{M_{\updownarrow}}$, the proven closed-form solution to the lower bound $\tilde{P}_{\mathcal{M}, t}^{\it LB}(s' \mid s, a)$ is:

\[\tilde{P}_{\mathcal{M}, t}^{\it LB}(s' \mid s, a) = 
\begin{cases}
\dfrac{P_\mathcal{M}(s' \mid s, a) - (1 - P_\mathcal{M}(s_{t+1} \mid s_t, a_t))}{P_\mathcal{M}(s_{t+1} \mid s_t, a_t)} & \text{if $P_\mathcal{M}(s' \mid s, a) > 1 - P_\mathcal{M}(s_{t+1} \mid s_t, a_t)$}\\
0 & \text{otherwise}\\
\end{cases}
\]

Therefore, to identify the lower bound, we need to minimise this over all MDPs $\mathcal{M} \in \mathcal{M}_\updownarrow$. Clearly, for any MDP $\mathcal{M}$ where $P_\mathcal{M}(s' \mid s, a) > 1 - P_\mathcal{M}(s_{t+1} \mid s_t, a_t)$: \[\dfrac{P_\mathcal{M}(s' \mid s, a) - (1 - P_\mathcal{M}(s_{t+1} \mid s_t, a_t))}{P_\mathcal{M}(s_{t+1} \mid s_t, a_t)} > 0\]
Therefore, the absolute minimum value of $\tilde{P}_{\mathcal{M}, t}^{\it LB}(s' \mid s, a)$ across all MDPs cannot be less than $0$. Consider whether $P^{LB}_{\mathcal{M}_\updownarrow}(s' \mid s, a) > 1 - P^{LB}_{\mathcal{M}_\updownarrow}(s_{t+1} \mid s_t, a_t)$:

\begin{enumerate}
    \item If $P^{LB}_{\mathcal{M}_\updownarrow}(s' \mid s, a) \leq 1 - P^{LB}_{\mathcal{M}_\updownarrow}(s_{t+1} \mid s_t, a_t)$, then, because of Assumption \ref{ass: normalisation}, there must exist an MDP $\mathcal{M}$ where $P_\mathcal{M}(s_{t+1} \mid s_t, a_t) = P^{\it LB}_{\mathcal{M}_\updownarrow}(s_{t+1} \mid s_t, a_t)$ and $P_\mathcal{M}(s' \mid s, a) = P^{\it LB}_{\mathcal{M}_\updownarrow}(s' \mid s, a)$. Because $P^{LB}_{\mathcal{M}_\updownarrow}(s' \mid s, a) \leq 1 - P^{LB}_{\mathcal{M}_\updownarrow}(s_{t+1} \mid s_t, a_t)$, we know that in this MDP $\mathcal{M}$, the lower bound counterfactual probability $\tilde{P}_{\mathcal{M}, t}^{\it LB}(s' \mid s, a) = 0$. Since this is the minimum possible counterfactual probability across all MDPs, this means that $\tilde{P}_{\mathcal{M}_\updownarrow, t}^{\it LB}(s' \mid s, a) = 0$.
    
    \item Otherwise, if $P^{LB}_{\mathcal{M}_\updownarrow}(s' \mid s, a) > 1 - P^{LB}_{\mathcal{M}_\updownarrow}(s_{t+1} \mid s_t, a_t)$, then, for every MDP $\mathcal{M} \in \mathcal{M}_\updownarrow$, $P_\mathcal{M}(s' \mid s, a) > 1 - P_\mathcal{M}(s_{t+1} \mid s_t, a_t)$. Therefore, for every MDP, the lower bound counterfactual probability for the transition will be $\tilde{P}_{\mathcal{M}, t}(s' \mid s, a) = \dfrac{P_\mathcal{M}(s' \mid s, a) - (1 - P_\mathcal{M}(s_{t+1} \mid s_t, a_t))}{P_\mathcal{M}(s_{t+1} \mid s_t, a_t)}$.
    
    Therefore, the minimum lower bound counterfactual probability will be:

    \begin{align*}
    \tilde{P}^{LB}_{\mathcal{M}_\updownarrow}(s' \mid s, a) &= {\min}_{\mathcal{M} \in \mathcal{M_\updownarrow}} \left( \tilde{P}^{LB}_{\mathcal{M},t}(s' \mid s, a) \right) \\
        &= {\min}_{\mathcal{M} \in \mathcal{M_\updownarrow}} \left( \dfrac{P_{\mathcal{M}}(s' \mid s, a) - (1 - P_{\mathcal{M}}(s_{t+1} \mid s_t, a_t))}{P_{\mathcal{M}}(s_{t+1} \mid s_t, a_t)}  \right) \\
        &= {\min}_{\mathcal{M} \in \mathcal{M_\updownarrow}} \left( \dfrac{P_{\mathcal{M}}(s' \mid s, a) - 1 + P_{\mathcal{M}}(s_{t+1} \mid s_t, a_t)}{P_{\mathcal{M}}(s_{t+1} \mid s_t, a_t)}  \right) \\
        &= {\min}_{\mathcal{M} \in \mathcal{M_\updownarrow}} \left( 1 - \left(\dfrac{1-P_{\mathcal{M}}(s' \mid s, a)}{P_{\mathcal{M}}(s_{t+1} \mid s_t, a_t)}\right)  \right)
    \end{align*}

    Clearly, this formula will be minimal when $P_\mathcal{M}(s' \mid s, a) = P^{LB}_{\mathcal{M}_\updownarrow}(s' \mid s, a)$ and $P_\mathcal{M}(s_{t+1} \mid s_t, a_t) = P^{LB}_{\mathcal{M}_\updownarrow}(s_{t+1} \mid s_t, a_t)$. Because of Assumption \ref{ass: normalisation}, there must exist an MDP $\mathcal{M}$ where $P_\mathcal{M}(s' \mid s, a) = P^{LB}_{\mathcal{M}_\updownarrow}(s' \mid s, a)$ and $P_\mathcal{M}(s_{t+1} \mid s_t, a_t) = P^{LB}_{\mathcal{M}_\updownarrow}(s_{t+1} \mid s_t, a_t)$. Also, note that because of Assumption \ref{ass: observed transition}, $P^{\it LB}_{\mathcal{M}_\updownarrow}(s_{t+1} \mid s_t, a_t) > 0$.
    
    Therefore, the minimum lower bound counterfactual probability in this case is:

    \[
    \tilde{P}^{\it LB}_{\mathcal{M}_\updownarrow}(s' \mid s, a) = \dfrac{P^{\it LB}_{\mathcal{M}_\updownarrow}(s' \mid s, a) - (1 - P^{\it LB}_{\mathcal{M}_\updownarrow}(s_{t+1} \mid s_t, a_t))}{P^{\it LB}_{\mathcal{M}_\updownarrow}(s_{t+1} \mid s_t, a_t)} 
    \]
\end{enumerate}

We can combine these cases as follows:
\[\tilde{P}_{\mathcal{M}_\updownarrow, t}^{\it LB}(s' \mid s, a) = 
\begin{cases}
\dfrac{P^{\it LB}_{\mathcal{M}_\updownarrow}(s' \mid s, a) - (1 - P^{\it LB}_{\mathcal{M}_\updownarrow}(s_{t+1} \mid s_t, a_t))}{P^{\it LB}_{\mathcal{M}_\updownarrow}(s_{t+1} \mid s_t, a_t)} & \text{if $P^{\it LB}_{\mathcal{M}_\updownarrow}(s' \mid s, a) > 1 - P^{\it LB}_{\mathcal{M}_\updownarrow}(s_{t+1} \mid s_t, a_t)$}\\
0 & \text{otherwise}\\
\end{cases}
\]
\end{proof}
\subsection{Overlapping Support $(s, a)$}

\subsubsection{Upper Bound}
\begin{theorem}
\label{theorem: upper overlapping imdp}
Assume $(s, a) \neq (s_t, a_t)$, and assume that the support of $(s, a)$ overlaps with the support of $(s_t, a_t)$ in the IMDP $\mathcal{M}_\updownarrow$. In this case, the optimisation problem in \eqref{eq: optimisation ub} will produce an upper bound of:
\[
\tilde{P}_{\mathcal{M}_\updownarrow, t}^{\it UB}(s' \mid s, a) = \\
\]
\begin{align*}
&\begin{cases}
\min(1, \dfrac{P^{\it UB}_{\mathcal{M}_\updownarrow}(s_{t+1} \mid s, a))}{P^{\it LB}_{\mathcal{M}_\updownarrow}(s_{t+1} \mid s_t, a_t)} & \text{if $s' = s_{t+1}$}\\
0 & \text{if $P^{\it LB}_{\mathcal{M}_\updownarrow}(s' \mid s_t, a_t) > 0$ and} \\ & \text{$\dfrac{P^{\it LB}_{\mathcal{M}_\updownarrow}(s_{t+1} \mid s, a)}{P^{\it UB}_{\mathcal{M}_\updownarrow}(s_{t+1} \mid s_t, a_t)} \geq \dfrac{P^{\it UB}_{\mathcal{M}_\updownarrow}(s' \mid s, a)}{P^{\it LB}_{\mathcal{M}_\updownarrow}(s' \mid s_t, a_t)}$}\\
\min(P^{\it UB}_{\mathcal{M}_\updownarrow}(s' \mid s, a), 1 - P^{\it LB}_{\mathcal{M}_\updownarrow}(s_{t+1} \mid s, a)) & \text{if $P^{LB}_{\mathcal{M}_\updownarrow}(s' \mid s_t, a_t) > 0$}\\
\min\left(1 - P^{LB}_{\mathcal{M}_\updownarrow}(s_{t+1} \mid s, a), \dfrac{P^{UB}_{\mathcal{M}_\updownarrow}(s' \mid s, a)}{P^{LB}_{\mathcal{M}_\updownarrow}(s_{t+1} \mid s_t, a_t)}\right) & \text{if $P^{UB}_{\mathcal{M}_\updownarrow}(s' \mid s_t, a_t) = 0$}\\
\min\left(1 - P^{LB}_{\mathcal{M}_\updownarrow}(s_{t+1} \mid s, a),  \dfrac{P^{UB}_{\mathcal{M}_\updownarrow}(s' \mid s, a)}{\max\left(\substack{1 - \sum_{s'' \in \mathcal{S} \setminus \{s', s_{t+1}\}}P^{UB}_{\mathcal{M}_\updownarrow}(s'' \mid s_t, a_t), \\P^{LB}(s_{t+1} \mid s_t, a_t)}\right)}\right) & \text{otherwise}
\end{cases}
\end{align*}
\end{theorem}

\begin{proof}
For any MDP $\mathcal{M} \in \mathcal{M}_\updownarrow$, the proven closed-form solution to the upper bound $\tilde{P}_{\mathcal{M}, t}(s' \mid s, a)$ is as follows: 

\[\tilde{P}_{\mathcal{M}, t}(s' \mid s, a) =
\begin{cases}
\dfrac{\min(P_\mathcal{M}(s_{t+1} \mid s_t, a_t), P_\mathcal{M}(s_{t+1} \mid s, a))}{P_\mathcal{M}(s_{t+1} \mid s_t, a_t)} & \text{if $s' = s_{t+1}$}\\
0 & \text{if $P_\mathcal{M}(s' \mid s_t, a_t) > 0$} \\ & \text{and $\dfrac{P_\mathcal{M}(s_{t+1} \mid s, a)}{P_\mathcal{M}(s_{t+1} \mid s_t, a_t)}\geq\dfrac{P_\mathcal{M}(s' \mid s, a)}{P_\mathcal{M}(s' \mid s_t, a_t)}$}\\
\min(P_\mathcal{M}(s' \mid s, a), 1 - P_\mathcal{M}(s_{t+1} \mid s, a)) & \text{if $P_\mathcal{M}(s' \mid s_t, a_t) > 0$}\\
\min\left(1 - P_\mathcal{M}(s_{t+1} \mid s, a), \dfrac{P_\mathcal{M}(s' \mid s, a)}{P_\mathcal{M}(s_{t+1} \mid s_t, a_t)}\right) & \text{otherwise}\\
\end{cases}
\]

Consider the following disjoint cases.

\paragraph{Case 1: $s' = s_{t+1}$}
In this case, for any MDP $\mathcal{M} \in \mathcal{M}_\updownarrow$, the closed-form solution is:

\[\tilde{P}_{\mathcal{M},t}^{\it UB}(s_{t+1} \mid s, a) =
\dfrac{\min(P_\mathcal{M}(s_{t+1} \mid s_t, a_t), P_\mathcal{M}(s_{t+1} \mid s, a))}{P_\mathcal{M}(s_{t+1} \mid s_t, a_t)}\]

Therefore, we need to maximise this over all MDPs $\mathcal{M} \in \mathcal{M}_\updownarrow$. Clearly, for any MDP $\mathcal{M}$ where $P_\mathcal{M}(s_{t+1} \mid s, a) < P_\mathcal{M}(s_{t+1} \mid s_t, a_t)$, $\dfrac{P_\mathcal{M}(s_{t+1} \mid s, a)}{P_\mathcal{M}(s_{t+1} \mid s_t, a_t)} < 1$, so the maximum possible value of $\tilde{P}_{\mathcal{M}, t}^{\it UB}(s' \mid s, a)$ across all MDPs is $1$. Consider now whether $P^{UB}_{\mathcal{M}_\updownarrow}(s_{t+1} \mid s, a) < P^{LB}_{\mathcal{M}_\updownarrow}(s_{t+1} \mid s_t, a_t)$:

\begin{enumerate}
    \item If $P^{UB}_{\mathcal{M}_\updownarrow}(s_{t+1} \mid s, a) \geq P^{LB}_{\mathcal{M}_\updownarrow}(s_{t+1} \mid s_t, a_t)$, then, because of Assumption \ref{ass: normalisation}, we know there exists an MDP $\mathcal{M} \in \mathcal{M}_\updownarrow$ where $P_\mathcal{M}(s_{t+1} \mid s, a) = P^{UB}_{\mathcal{M}_\updownarrow}(s_{t+1} \mid s, a)$ and $P_\mathcal{M}(s_{t+1} \mid s_t, a_t) = P^{LB}_{\mathcal{M}_\updownarrow}(s_{t+1} \mid s_t, a_t)$. Also, note that because of Assumption \ref{ass: observed transition}, $P^{\it LB}_{\mathcal{M}_\updownarrow}(s_{t+1} \mid s_t, a_t) > 0$. In this MDP $\mathcal{M}$, we know $P_\mathcal{M}(s_{t+1} \mid s, a) \geq P_\mathcal{M}(s_{t+1} \mid s_t, a_t)$, therefore, the upper bound counterfactual probability $\tilde{P}_{\mathcal{M}, t}(s' \mid s, a) = \dfrac{P_\mathcal{M}(s_{t+1} \mid s_t, a_t)}{P_\mathcal{M}(s_{t+1} \mid s_t, a_t)} = 1$, and since this is the maximum possible counterfactual probability, $\tilde{P}_{\mathcal{M}_\updownarrow,t}^{\it UB}(s_{t+1} \mid s, a) = 1$.

    \item Otherwise, if $P^{UB}_{\mathcal{M}_\updownarrow}(s_{t+1} \mid s, a) < P^{LB}_{\mathcal{M}_\updownarrow}(s_{t+1} \mid s_t, a_t)$, then, for every MDP $\mathcal{M} \in \mathcal{M}_\updownarrow$, $P_{\mathcal{M}}(s_{t+1} \mid s, a) < P_{\mathcal{M}}(s_{t+1} \mid s_t, a_t)$. Therefore, for every MDP, the upper bound counterfactual probability for the transition will be $\tilde{P}_{\mathcal{M}, t}^{\it UB}(s' \mid s, a) = 
    \dfrac{P_\mathcal{M}(s_{t+1} \mid s, a)}{P_\mathcal{M}(s_{t+1} \mid s_t, a_t)}$. The maximum upper bound counterfactual probability, therefore, will be:

    \begin{align*}
    \tilde{P}^{UB}_{\mathcal{M}_\updownarrow}(s_{t+1} \mid s, a) &= {\max}_{\mathcal{M} \in \mathcal{M_\updownarrow}} \left( \tilde{P}^{UB}_{\mathcal{M}, t}(s_{t+1} \mid s, a) \right) \\
    &= {\max}_{\mathcal{M} \in \mathcal{M_\updownarrow}} \left( \dfrac{P_{\mathcal{M}}(s_{t+1} \mid s, a)}{P_{\mathcal{M}}(s_{t+1} \mid s_t, a_t)}\right) \\
    &= \dfrac{P^{\it UB}_{\mathcal{M}_\updownarrow}(s_{t+1} \mid s, a)}{P^{\it LB}_{\mathcal{M}_\updownarrow}(s_{t+1} \mid s_t, a_t)}
    \end{align*}

    Because of Assumption \ref{ass: normalisation}, there must exist an MDP $\mathcal{M}$ where $P_\mathcal{M}(s_{t+1} \mid s_t, a_t) = P^{\it LB}_{\mathcal{M}_\updownarrow}(s_{t+1} \mid s_t, a_t)$ and $P_\mathcal{M}(s' \mid s, a) = P^{\it UB}_{\mathcal{M}_\updownarrow}(s' \mid s, a)$. Also, note that because of Assumption \ref{ass: observed transition}, $P^{\it LB}_{\mathcal{M}_\updownarrow}(s_{t+1} \mid s_t, a_t) > 0$.
    
    Therefore, the maximum upper bound counterfactual probability in this case is 
    \[
    \tilde{P}^{UB}_{\mathcal{M}_\updownarrow}(s' \mid s, a) = \dfrac{P^{\it UB}_{\mathcal{M}_\updownarrow}(s' \mid s, a)}{P^{\it LB}_{\mathcal{M}_\updownarrow}(s_{t+1} \mid s_t, a_t)}
    \]
\end{enumerate}

We can combine these cases as follows:
\[
\tilde{P}_{\mathcal{M}_\updownarrow, t}^{\it UB}(s_{t+1} \mid s, a) = 
\min(1, \dfrac{P^{\it UB}_{\mathcal{M}_\updownarrow}(s_{t+1} \mid s, a))}{P^{\it LB}_{\mathcal{M}_\updownarrow}(s_{t+1} \mid s_t, a_t)}
\]

\paragraph{Case 2: $P^{\it LB}_{\mathcal{M}_\updownarrow}(s' \mid s_t, a_t) > 0$ and $\dfrac{P^{\it LB}_{\mathcal{M}_\updownarrow}(s_{t+1} \mid s, a)}{P^{\it UB}_{\mathcal{M}_\updownarrow}(s_{t+1} \mid s_t, a_t)} \geq \dfrac{P^{\it UB}_{\mathcal{M}_\updownarrow}(s' \mid s, a)}{P^{\it LB}_{\mathcal{M}_\updownarrow}(s' \mid s_t, a_t)}$}

This is a valid condition, because, by assumption, $P^{LB}(s' \mid s_t, a_t) > 0$, and $P^{LB}(s_{t+1} \mid s_t, a_t) > 0$ since this was the observed transition (Assumption \ref{ass: observed transition}).

If this condition holds, then, for every $\mathcal{M} \in \mathcal{M}_{\updownarrow}$, $\dfrac{P_{\mathcal{M}}(s_{t+1} \mid s, a)}{P_{\mathcal{M}}(s_{t+1} \mid s_t, a_t)}\geq\dfrac{P_{\mathcal{M}}(s' \mid s, a)}{P_{\mathcal{M}}(s' \mid s_t, a_t)}$, and so, for every $\mathcal{M} \in \mathcal{M}_{\updownarrow}$ the upper bound counterfactual probability must be $0$, to satisfy the counterfactual stability constraint. Therefore, we know the maximum upper bound counterfactual probability $\tilde{P}^{UB}_{\updownarrow}(s' \mid s, a)$ will be:

\begin{align*}
\tilde{P}^{UB}_{\mathcal{M}_\updownarrow}(s' \mid s, a) &= {\max}_{\mathcal{M} \in \mathcal{M_\updownarrow}} \left( \tilde{P}^{UB}_{\mathcal{M}}(s' \mid s, a) \right) \\
&= {\max}_{\mathcal{M} \in \mathcal{M_\updownarrow}} \left( 0 \right) \\
&= 0
\end{align*}

\paragraph{Case 3: $P^{LB}_{\mathcal{M}_\updownarrow}(s' \mid s_t, a_t) > 0$}
If $P^{LB}(s' \mid s_t, a_t) > 0$, then, for all MDPs $\mathcal{M} \in \mathcal{M}_{\updownarrow}$, the counterfactual probability of the transition $\tilde{P}_{\mathcal{M},t}(s' \mid s, a) \leq P_\mathcal{M}(s' \mid s, a)$, due to the counterfactual monotonicity constraints. 

For every MDP $\mathcal{M} \in \mathcal{M}_{\updownarrow}$, the counterfactual probability $\tilde{P}_{\mathcal{M},t}^{\it UB}(s' \mid s, a) =
\min(P_{\mathcal{M}}(s' \mid s, a), 1 - P_{\mathcal{M}}(s_{t+1} \mid s, a))$, as long as the counterfactual stability constraint does not constrain the counterfactual probability to $0$. This leads to the following optimisation problem: 

\begin{align*}
\tilde{P}^{UB}_{\mathcal{M}_\updownarrow, t}(s' \mid s, a) &= {\max}_{\mathcal{M} \in \mathcal{M_\updownarrow}} \left( \tilde{P}^{UB}_{\mathcal{M},t}(s' \mid s, a) \right) \\
&= {\max}_{\mathcal{M} \in \mathcal{M_\updownarrow}} \left( \min(P_{\mathcal{M}}(s' \mid s, a), 1 - P_{\mathcal{M}}(s_{t+1} \mid s, a)) \right) \\
\end{align*}

Because of Assumption \ref{ass: normalisation}, there must exist an MDP $\mathcal{M} \in \mathcal{M}_\updownarrow$ with 
\begin{itemize}
    \item $P_{\mathcal{M}}(s' \mid s, a) = P^{\it UB}_{\mathcal{M}_\updownarrow}(s' \mid s, a)$
    \item $P_{\mathcal{M}}(s_{t+1} \mid s, a) = P^{\it LB}_{\mathcal{M}_\updownarrow}(s_{t+1} \mid s, a)$
    \item $P_{\mathcal{M}}(s' \mid s_t, a_t) = P^{\it LB}_{\mathcal{M}_\updownarrow}(s' \mid s_t, a_t)$
    \item $P_{\mathcal{M}}(s_{t+1} \mid s_t, a_t) = P^{\it UB}_{\mathcal{M}_\updownarrow}(s_{t+1} \mid s_t, a_t)$
\end{itemize}

This MDP results in the maximum possible values of $P_{\mathcal{M}}(s' \mid s, a)$ and $1 - P_{\mathcal{M}}(s_{t+1} \mid s, a))$, therefore, this also produces the maximum possible value of $\min(P_{\mathcal{M}}(s' \mid s, a), 1 - P_{\mathcal{M}}(s_{t+1} \mid s, a))$.

In this case, we know $\dfrac{P^{\it LB}_{\mathcal{M}_\updownarrow}(s_{t+1} \mid s, a)}{P^{\it UB}_{\mathcal{M}_\updownarrow}(s_{t+1} \mid s_t, a_t)} < \dfrac{P^{\it UB}_{\mathcal{M}_\updownarrow}(s' \mid s, a)}{P^{\it LB}_{\mathcal{M}_\updownarrow}(s' \mid s_t, a_t)}$ (otherwise, we would be in Case 2), therefore for this MDP $\mathcal{M}$, the counterfactual stability constraint will not constrain the counterfactual transition probability $\tilde{P}_{\mathcal{M},t}(s' \mid s, a)$. Therefore, the maximum counterfactual probability across all MDPs $\mathcal{M} \in \mathcal{M}_\updownarrow$ will be:

\[
\tilde{P}^{UB}_{\mathcal{M}_\updownarrow, t}(s' \mid s, a) = \min(P^{\it UB}_{\mathcal{M}_\updownarrow}(s' \mid s, a), 1 - P^{\it LB}_{\mathcal{M}_\updownarrow}(s_{t+1} \mid s, a))
\]

\paragraph{Case 4: $P^{UB}_{\mathcal{M}_\updownarrow}(s' \mid s_t, a_t) = 0$} If $P^{UB}_{\mathcal{M}_\updownarrow}(s' \mid s_t, a_t) = 0$, then, for all MDPs $\mathcal{M} \in \mathcal{M}_\updownarrow$, $P^{UB}_{\mathcal{M}}(s' \mid s_t, a_t) = 0$. Therefore, the upper bound counterfactual probability of the transition for all MDPs $\mathcal{M} \in \mathcal{M}_\updownarrow$ is as follows:

\[\tilde{P}_{\mathcal{M},t}^{\it UB}(s' \mid s, a) =
\min\left(1 - P_\mathcal{M}(s_{t+1} \mid s, a), \dfrac{P_\mathcal{M}(s' \mid s, a)}{P_\mathcal{M}(s_{t+1} \mid s_t, a_t)}\right)\]

This results in the following optimisation problem:

\begin{align*}
\tilde{P}^{UB}_{\mathcal{M}_\updownarrow, t}(s' \mid s, a) &= {\max}_{\mathcal{M} \in \mathcal{M_\updownarrow}} \left( \tilde{P}^{UB}_{\mathcal{M},t}(s' \mid s, a) \right) \\
&= {\max}_{\mathcal{M} \in \mathcal{M_\updownarrow}} \left( \min\left(1 - P_{\mathcal{M}}(s_{t+1} \mid s, a), \dfrac{P_{\mathcal{M}}(s' \mid s, a)}{P_{\mathcal{M}}(s_{t+1} \mid s_t, a_t)}\right) \right) \\
\end{align*}

Because of Assumption \ref{ass: normalisation}, there must exist an MDP $\mathcal{M} \in \mathcal{M}_\updownarrow$ where:
\begin{itemize}
    \item $P_{\mathcal{M}}(s_{t+1} \mid s, a) = P^{LB}_{\mathcal{M}_\updownarrow}(s_{t+1} \mid s, a)$
    \item $P_{\mathcal{M}}(s' \mid s, a) = P^{UB}_{\mathcal{M}_\updownarrow}(s' \mid s, a)$
    \item $P_{\mathcal{M}}(s_{t+1} \mid s_t, a_t) = P^{LB}_{\mathcal{M}_\updownarrow}(s_{t+1} \mid s_t, a_t)$
    
\end{itemize}

Also, note that because of Assumption \ref{ass: observed transition}, $P^{\it LB}_{\mathcal{M}_\updownarrow}(s_{t+1} \mid s_t, a_t) > 0$.

This MDP results in the maximum possible values of $1 - P_{\mathcal{M}}(s_{t+1} \mid s, a)$ and $\dfrac{P_{\mathcal{M}}(s' \mid s, a)}{P_{\mathcal{M}}(s_{t+1} \mid s_t, a_t)}$, and therefore produces the maximum possible value of $\min\left(1 - P_{\mathcal{M}}(s_{t+1} \mid s, a), \dfrac{P_{\mathcal{M}}(s' \mid s, a)}{P_{\mathcal{M}}(s_{t+1} \mid s_t, a_t)}\right)$.

Therefore, the maximum counterfactual probability across all MDPs $\mathcal{M} \in \mathcal{M}_\updownarrow$ will be:

\[\tilde{P}_{\mathcal{M}_\updownarrow, t}^{\it UB}(s' \mid s, a) =
\min\left(1 - P^{LB}_{\mathcal{M}_\updownarrow}(s_{t+1} \mid s, a), \dfrac{P^{UB}_{\mathcal{M}_\updownarrow}(s' \mid s, a)}{P^{LB}_{\mathcal{M}_\updownarrow}(s_{t+1} \mid s_t, a_t)}\right)\]

\paragraph{Case 5: $P^{LB}_{\mathcal{M}_\updownarrow}(s' \mid s_t, a_t) = 0$ and $P^{UB}_{\mathcal{M}_\updownarrow}(s' \mid s_t, a_t) > 0$}

For every MDP $\mathcal{M} \in \mathcal{M}_\updownarrow$, the counterfactual probability of $\tilde{P}_{\mathcal{M}, t}(s' \mid s, a)$ depends on the value of $P_\mathcal{M}(s' \mid s_t, a_t)$:
\begin{itemize}
    \item If $P_\mathcal{M}(s' \mid s_t, a_t) >0$, then the counterfactual probability of the transition will either be $\tilde{P}_{\mathcal{M},t}(s' \mid s, a) = \min(P_{\mathcal{M}}(s' \mid s, a), 1 - P_\mathcal{M}(s_{t+1} \mid s, a))$, because of the counterfactual monotonicity constraint, or $0$ if the counterfactual stability constraint holds.

    \item Otherwise, if $P_\mathcal{M}(s' \mid s_t, a_t) =0$, then the counterfactual probability of the transition $\tilde{P}_{\mathcal{M},t}(s' \mid s, a) = \min\left(1 - P_\mathcal{M}(s_{t+1} \mid s, a), \dfrac{P_\mathcal{M}(s' \mid s, a)}{P_\mathcal{M}(s_{t+1} \mid s_t, a_t)}\right)$.
\end{itemize}

Because $\forall \mathcal{M} \in \mathcal{M}_\updownarrow, 0 < P_{\mathcal{M}}(s_{t+1} \mid s_t, a_t) \leq 1$ we have, for all MDPs $\mathcal{M} \in \mathcal{M}_\updownarrow$: 

\[\min\left(1 - P_\mathcal{M}(s_{t+1} \mid s, a), \dfrac{P_\mathcal{M}(s' \mid s, a)}{P_\mathcal{M}(s_{t+1} \mid s_t, a_t)}\right) \geq \min(P_\mathcal{M}(s' \mid s, a), 1 - P_\mathcal{M}(s_{t+1} \mid s, a))\]

and, trivially,

\[\min\left(1 - P_\mathcal{M}(s_{t+1} \mid s, a), \dfrac{P_\mathcal{M}(s' \mid s, a)}{P_\mathcal{M}(s_{t+1} \mid s_t, a_t)}\right) \geq 0\]

Therefore, the counterfactual transition probability for MDPs where $P_\mathcal{M}(s' \mid s_t, a_t) =0$ will always be greater than or equal to the counterfactual transition probability when $P_\mathcal{M}(s' \mid s_t, a_t) > 0$. Because of Assumption \ref{ass: normalisation}, and since $P^{LB}_{\mathcal{M}_\updownarrow}(s' \mid s_t, a_t) = 0$, there must exist an MDP $\mathcal{M} \in \mathcal{M}_\updownarrow$ where:
    \begin{itemize}
        \item $P_{\mathcal{M}}(s' \mid s_t, a_t) = P^{LB}_{\mathcal{M}_\updownarrow}(s' \mid s_t, a_t) = 0$
        \item $P_\mathcal{M}(s_{t+1} \mid s_t, a_t) = \max\left(1 - \sum_{s'' \in \mathcal{S} \setminus \{s', s_{t+1}\}}P^{UB}_{\mathcal{M}_\updownarrow}(s'' \mid s_t, a_t), P^{LB}(s_{t+1} \mid s_t, a_t)\right)$
        \item $P_{\mathcal{M}}(s' \mid s, a) = P^{\it UB}_{\mathcal{M}_\updownarrow}(s' \mid s, a)$
        \item $P_{\mathcal{M}}(s_{t+1} \mid s, a) = P^{LB}_{\mathcal{M}_\updownarrow}(s_{t+1} \mid s, a)$
    \end{itemize}

    This assignment leads to the maximum possible value of $\min\left(1 - P_{\mathcal{M}}(s_{t+1} \mid s, a), \dfrac{P_{\mathcal{M}}(s' \mid s, a)}{P_\mathcal{M}(s_{t+1} \mid s_t, a_t)}\right)$. Therefore, the maximum upper bound counterfactual probability is:

    \[
        \tilde{P}_{\mathcal{M}, t}^{\it UB}(s_{t+1} \mid s, a) = \min\left(1 - P^{LB}_{\mathcal{M}_\updownarrow}(s_{t+1} \mid s, a), \dfrac{P^{UB}_{\mathcal{M}_\updownarrow}(s' \mid s, a)}{\max\left(\substack{1 - \sum_{s'' \in \mathcal{S} \setminus \{s', s_{t+1}\}}P^{UB}_{\mathcal{M}_\updownarrow}(s'' \mid s_t, a_t), \\P^{LB}(s_{t+1} \mid s_t, a_t)}\right)}\right)
    \]

\end{proof}

\subsubsection{Lower Bound}
Because the lower bound closed-form solutions for each MDP $\mathcal{M}$ contain $\sum_{s' \in \mathcal{S}\setminus\{s'\}}{\tilde{P}_{\mathcal{M}, t}^{\it UB}(s' \mid s, a)}$, this makes it challenging to identify the minimum lower bound counterfactual probability across all MDPs $\mathcal{M}\in\mathcal{M}_\updownarrow$. Because of this, we provide two solutions: a closed-form solution for a non-tight lower bound, and a reduced optimisation problem that solves for the exact lower bound as efficiently as possible.

\paragraph{Solution 1: Non-tight Lower Bound}
\begin{theorem}
\label{theorem: lower overlapping imdp}
Assume $(s, a) \neq (s_t, a_t)$, and assume that the support of $(s, a)$ overlaps with the support of $(s_t, a_t)$ in the IMDP $\mathcal{M}_\updownarrow$. In this case, the optimisation problem in \eqref{eq: optimisation lb} will produce a lower bound that is guaranteed to be no smaller than:
\[\tilde{P}_{\mathcal{M}_\updownarrow, t}^{\it LB}(s' \mid s, a) = 
\begin{cases}
\max\left(P^{\it LB}_{\mathcal{M}_\updownarrow}(s' \mid s, a), 1 - \sum_{s' \in \mathcal{S}\setminus\{s_{t+1}\}}{\tilde{P}_{\mathcal{M}_\updownarrow, t}^{\it UB}(s' \mid s, a)}\right) \text{ if $s' = s_{t+1}$}\\
0 \text{ if $P^{\it UB}_{\mathcal{M}_\updownarrow}(s' \mid s_t, a_t) > 0$ and }\text{$\dfrac{P^{\it UB}_{\mathcal{M}_\updownarrow}(s_{t+1} \mid s, a)}{P^{\it LB}_{\mathcal{M}_\updownarrow}(s_{t+1} \mid s_t, a_t)} \geq \dfrac{P^{\it LB}_{\mathcal{M}_\updownarrow}(s' \mid s, a)}{P^{\it UB}_{\mathcal{M}_\updownarrow}(s' \mid s_t, a_t)}$}\\
\max\left(0, 1 - \sum_{s' \in \mathcal{S}\setminus\{s'\}}{\tilde{P}_{\mathcal{M}_\updownarrow, t}^{\it UB}(s' \mid s, a)}\right) \text{ otherwise}\\
\end{cases}
\]
\end{theorem}

Consider the following disjoint cases.

\paragraph{Case 1: $s' = s_{t+1}$}
Because of Assumption \ref{ass: observed transition}, for any MDP $\mathcal{M} \in \mathcal{M}_\updownarrow$, $P_\mathcal{M}(s_{t+1} \mid s_t, a_t) > 0$ since this was the observed transition. Therefore, for any MDP $\mathcal{M} \in \mathcal{M}_\updownarrow$, the monotonicity assumption ensures that $\tilde{P}_{\mathcal{M},t}(s_{t+1} \mid s, a) \geq P_\mathcal{M}(s_{t+1} \mid s, a)$, and the closed-form solution is:

\[
\tilde{P}_{\mathcal{M},t}^{\it LB}(s_{t+1} \mid s, a) = 
\max\left(P_\mathcal{M}(s_{t+1} \mid s, a), 1 - \sum_{s'' \in \mathcal{S}\setminus\{s_{t+1}\}}{\tilde{P}_{\mathcal{M},t}^{\it UB}(s'' \mid s, a)}\right)
\]

We need to minimise this condition over all MDPs $\mathcal{M} \in \mathcal{M}_\updownarrow$. We can immediately see that the minimum values of each part of the closed-form solution, i.e., $P_\mathcal{M}(s_{t+1} \mid s, a)$ and \\$1 - \sum_{s'' \in \mathcal{S}\setminus\{s_{t+1}\}}{\tilde{P}_{\mathcal{M},t}^{\it UB}(s'' \mid s, a)}$ will occur in the same MDP $\mathcal{M}$. This is because, to minimise $1 - \sum_{s'' \in \mathcal{S}\setminus\{s_{t+1}\}}{\tilde{P}_{\mathcal{M},t}^{\it UB}(s'' \mid s, a)}$, we need to identify an MDP $\mathcal{M}$ which maximises the value of $\sum_{s'' \in \mathcal{S}\setminus\{s_{t+1}\}}{{P}_{\mathcal{M},t}(s'' \mid s, a)}$. Clearly, the maximum value of $\sum_{s'' \in \mathcal{S}\setminus\{s_{t+1}\}}{\tilde{P}_{\mathcal{M},t}(s'' \mid s, a)}$ is $1-P^{\it LB}_{\mathcal{M}_\updownarrow}(s_{t+1} \mid s, a)$, and the minimum value of $P_\mathcal{M}(s_{t+1} \mid s, a)$ across all MDPs $\mathcal{M}\in\mathcal{M}_\updownarrow$ is $P^{\it LB}_{\mathcal{M}_\updownarrow}(s_{t+1} \mid s, a)$. Therefore, we can consider each part of this closed-form solution separately.

Now, consider the minimum value of $1 - \sum_{s'' \in \mathcal{S}\setminus\{s_{t+1}\}}{\tilde{P}_{\mathcal{M},t}^{\it UB}(s'' \mid s, a)}$ across all MDPs $\mathcal{M}\in\mathcal{M}_\updownarrow$. In any MDP $\mathcal{M} \in \mathcal{M}_\updownarrow$, we require the transition probabilities from every state-action pair to sum to $1$. Because of this, it is challenging to determine the MDP with the greatest value of $\sum_{s'' \in \mathcal{S}\setminus\{s_{t+1}\}}{\tilde{P}_{\mathcal{M},t}^{\it UB}(s'' \mid s, a)}$. However, to identify a non-tight lower bound, we can ignore this constraint, and instead find the value of $\sum_{s'' \in \mathcal{S}\setminus\{s_{t+1}\}}{\tilde{P}_{\mathcal{M}_\updownarrow,t}^{\it UB}(s'' \mid s, a)}$ across the whole IMDP $\mathcal{M}_\updownarrow$. Clearly, $\sum_{s'' \in \mathcal{S}\setminus\{s_{t+1}\}}{\tilde{P}_{\mathcal{M}_\updownarrow,t}^{\it UB}(s'' \mid s, a)} \geq \sum_{s'' \in \mathcal{S}\setminus\{s_{t+1}\}}{\tilde{P}_{\mathcal{M}, t}^{\it UB}(s'' \mid s, a)}$ for any MDP $\mathcal{M} \in \mathcal{M}_\updownarrow$, so $1-\sum_{s'' \in \mathcal{S}\setminus\{s_{t+1}\}}{\tilde{P}_{\mathcal{M}_\updownarrow,t}^{\it UB}(s'' \mid s, a)}$ results in a lower bound.

Therefore, a minimum (but not necessarily tight) lower bound would be:
\[
\max\left(P^{\it LB}_{\mathcal{M}_\updownarrow}(s' \mid s, a), 1 - \sum_{s' \in \mathcal{S}\setminus\{s_{t+1}\}}{\tilde{P}_{\mathcal{M}_\updownarrow, t}^{\it UB}(s' \mid s, a)}\right)
\]

\paragraph{Case 2: $P^{\it UB}_{\mathcal{M}_\updownarrow}(s' \mid s_t, a_t) > 0$ and $\dfrac{P^{\it UB}_{\mathcal{M}_\updownarrow}(s_{t+1} \mid s, a)}{P^{\it LB}_{\mathcal{M}_\updownarrow}(s_{t+1} \mid s_t, a_t)} \geq \dfrac{P^{\it LB}_{\mathcal{M}_\updownarrow}(s' \mid s, a)}{P^{\it UB}_{\mathcal{M}_\updownarrow}(s' \mid s_t, a_t)}$}

This is a valid condition, because, by assumption, $P^{UB}_{\mathcal{M}_\updownarrow}(s' \mid s_t, a_t) > 0$, and $P^{UB}_{\mathcal{M}_\updownarrow}(s_{t+1} \mid s_t, a_t) \geq P^{LB}_{\mathcal{M}_\updownarrow}(s_{t+1} \mid s_t, a_t) > 0$ since this was the observed transition (Assumption \ref{ass: observed transition}).

Because of Assumption \ref{ass: normalisation}, there must exist an MDP $\mathcal{M} \in \mathcal{M}_\updownarrow$ with:

\begin{itemize}
    \item $P_\mathcal{M}(s' \mid s, a) = P^{\it LB}_{\mathcal{M}_\updownarrow}(s' \mid s, a)$
    \item $P_\mathcal{M}(s_{t+1} \mid s, a) = P^{\it UB}_{\mathcal{M}_\updownarrow}(s_{t+1} \mid s, a)$
    \item $P_\mathcal{M}(s' \mid s_t, a_t) = P^{\it UB}_{\mathcal{M}_\updownarrow}(s' \mid s_t, a_t)$
    \item $P_\mathcal{M}(s_{t+1} \mid s_t, a_t) = P^{\it LB}_{\mathcal{M}_\updownarrow}(s_{t+1} \mid s_t, a_t)$
\end{itemize}

Therefore, for this MDP $\mathcal{M}$, $\dfrac{P_{\mathcal{M}}(s_{t+1} \mid s, a)}{P_{\mathcal{M}}(s_{t+1} \mid s_t, a_t)}\geq\dfrac{P_{\mathcal{M}}(s' \mid s, a)}{P_{\mathcal{M}}(s' \mid s_t, a_t)}$, and so the lower bound counterfactual probability for $\tilde{P}^{\it LB}_{\mathcal{M},t} = 0$, to satisfy the counterfactual stability constraint. Since the counterfactual probability of any transition cannot be smaller than $0$, we know the minimum lower bound counterfactual probability across all MDPs $\mathcal{M} \in \mathcal{M}_\updownarrow$, $\tilde{P}^{LB}_{\mathcal{M}_\updownarrow,t}(s' \mid s, a)$ will be:

\begin{align*}
\tilde{P}^{LB}_{\mathcal{M}_\updownarrow,t}(s' \mid s, a) &= {\min}_{\mathcal{M} \in \mathcal{M_\updownarrow}} \left( \tilde{P}^{LB}_{\mathcal{M},t}(s' \mid s, a) \right) \\
&= {\min}_{\mathcal{M} \in \mathcal{M_\updownarrow}} \left( 0 \right) \\
&= 0
\end{align*}

\paragraph{Case 3: $P^{\it UB}_{\mathcal{M}_\updownarrow}(s' \mid s_t, a_t) > 0$ and $\dfrac{P^{\it UB}_{\mathcal{M}_\updownarrow}(s_{t+1} \mid s, a)}{P^{\it LB}_{\mathcal{M}_\updownarrow}(s_{t+1} \mid s_t, a_t)} < \dfrac{P^{\it LB}_{\mathcal{M}_\updownarrow}(s' \mid s, a)}{P^{\it UB}_{\mathcal{M}_\updownarrow}(s' \mid s_t, a_t)}$, or $P^{\it UB}_{\mathcal{M}_\updownarrow}(s' \mid s_t, a_t) = 0$}
In this case, because either $P^{\it UB}_{\mathcal{M}_\updownarrow}(s' \mid s_t, a_t) > 0$ and $\dfrac{P^{\it UB}_{\mathcal{M}_\updownarrow}(s_{t+1} \mid s, a)}{P^{\it LB}_{\mathcal{M}_\updownarrow}(s_{t+1} \mid s_t, a_t)} < \dfrac{P^{\it LB}_{\mathcal{M}_\updownarrow}(s' \mid s, a)}{P^{\it UB}_{\mathcal{M}_\updownarrow}(s' \mid s_t, a_t)}$, or $P^{\it UB}_{\mathcal{M}_\updownarrow}(s' \mid s_t, a_t) = 0$, then, for every MDP $\mathcal{M} \in \mathcal{M}_\updownarrow$, either $P_\mathcal{M}(s' \mid s_t, a_t) > 0$ and $\dfrac{P_{\mathcal{M}}(s_{t+1} \mid s, a)}{P_{\mathcal{M}}(s_{t+1} \mid s_t, a_t)}<\dfrac{P_{\mathcal{M}}(s' \mid s, a)}{P_{\mathcal{M}}(s' \mid s_t, a_t)}$, or $P_\mathcal{M}(s' \mid s_t, a_t) = 0$. Therefore, the closed-form solution for every MDP $\mathcal{M} \in \mathcal{M}_\updownarrow$ will be:
\[
\tilde{P}_{\mathcal{M},t}^{\it LB}(s' \mid s, a) = 
\max\left(0, 1 - \sum_{s'' \in \mathcal{S}\setminus\{s'\}}{\tilde{P}_{\mathcal{M},t}^{\it UB}(s'' \mid s, a)}\right)
\]

We need to minimise this condition over all MDPs $\mathcal{M} \in \mathcal{M}_\updownarrow$. Clearly, the counterfactual probability of any transition cannot be any smaller than $0$. Therefore, if there exists an MDP $\mathcal{M} \in \mathcal{M}_\updownarrow$ where $\tilde{P}^{\it LB}_{\mathcal{M},t}(s' \mid s, a) = 0$, then the absolute minimum value of $\tilde{P}^{\it LB}_{\mathcal{M}_\updownarrow,t}(s' \mid s, a) =0$ as well. The minimum lower bound counterfactual probability $\tilde{P}^{\it LB}_{\mathcal{M}_\updownarrow,t}(s' \mid s, a)$ will only be greater than $0$ if for every MDP $\mathcal{M} \in \mathcal{M}_\updownarrow$, $1 - \sum_{s'' \in \mathcal{S}\setminus\{s'\}}{\tilde{P}_{\mathcal{M},t}^{\it UB}(s'' \mid s, a)} > 0$. 

Now, consider the minimum value of $1-\sum_{s'' \in \mathcal{S}\setminus\{s_{t+1}\}}{\tilde{P}_{\mathcal{M},t}^{\it UB}(s'' \mid s, a)}$. In any MDP $\mathcal{M} \in \mathcal{M}_\updownarrow$, we require the transition probabilities from every state-action pair to sum to $1$. Because of this, it is challenging to determine the MDP with the greatest value of $\sum_{s'' \in \mathcal{S}\setminus\{s_{t+1}\}}{\tilde{P}_{\mathcal{M},t}^{\it UB}(s'' \mid s, a)}$. However, to identify a non-tight lower bound, we can ignore this constraint, and instead find the value of $\sum_{s'' \in \mathcal{S}\setminus\{s_{t+1}\}}{\tilde{P}_{\mathcal{M}_\updownarrow,t}^{\it UB}(s'' \mid s, a)}$ across the whole IMDP $\mathcal{M}_\updownarrow$. Clearly, $\sum_{s'' \in \mathcal{S}\setminus\{s_{t+1}\}}{\tilde{P}_{\mathcal{M}_\updownarrow,t}^{\it UB}(s'' \mid s, a)} \geq \sum_{s'' \in \mathcal{S}\setminus\{s_{t+1}\}}{\tilde{P}_{\mathcal{M}, t}^{\it UB}(s'' \mid s, a)}$ for any MDP $\mathcal{M} \in \mathcal{M}_\updownarrow$, so $1-\sum_{s'' \in \mathcal{S}\setminus\{s_{t+1}\}}{\tilde{P}_{\mathcal{M}_\updownarrow,t}^{\it UB}(s'' \mid s, a)}$ results in a lower bound.

Therefore, a minimum (but not necessarily tight) lower bound would be:
\[
\max\left(0, 1 - \sum_{s' \in \mathcal{S}\setminus\{s'\}}{\tilde{P}_{\mathcal{M}_\updownarrow, t}^{\it UB}(s' \mid s, a)}\right)
\]

\paragraph{Solution 2: Reducing the Optimisation Problem}
Alternatively, if it is computationally feasible, we can calculate the exact lower bound for the counterfactual probability of any transition by calculating the true minimum value of $1 - \sum_{s'' \in \mathcal{S}\setminus\{s'\}}{\tilde{P}_{\mathcal{M}_\updownarrow, t}^{\it UB}(s'' \mid s, a)}$, for Cases 1 and 3 in Solution 1. This can be solved by maximising the value of $\sum_{s'' \in \mathcal{S}\setminus\{s'\}}{\tilde{P}_{\mathcal{M}_\updownarrow, t}^{\it UB}(s'' \mid s, a)}$, in a nested optimisation problem, as follows:
\begin{align*}
    &\max_{\substack{P_{\mathcal{M}}(\cdot \mid s_t, a_t) \in P_{\mathcal{M}_\updownarrow}(\cdot \mid s_t, a_t) \\ P_{\mathcal{M}}(\cdot \mid s,a) \in P_{\mathcal{M}_\updownarrow}(\cdot \mid s,a)}} \sum_{s'' \in \mathcal{S}\setminus\{s'\}} {\tilde{P}_{t}^{\it UB}(s'' \mid s, a)} \\
    \text{s.t.} &\sum_{s' \in \mathcal{S}} P_{\mathcal{M}}(s' \mid s_t, a_t) = 1 \\
    &\sum_{s'\in \mathcal{S}} P_{\mathcal{M}}(s' \mid s,a) = 1 \\
    & P^{\it LB}_{\mathcal{M}_\updownarrow}(s' \mid s_t, a_t) \leq P_{\mathcal{M}}(s' \mid s_t, a_t) \leq P^{\it UB}_{\mathcal{M}_\updownarrow}(s' \mid s_t, a_t), \quad \forall s'\in \mathcal{S} \\
    & P^{\it LB}_{\mathcal{M}_\updownarrow}(s' \mid s,a) \leq P_{\mathcal{M}}(s' \mid s,a) \leq P^{\it UB}_{\mathcal{M}_\updownarrow}(s' \mid s,a), \quad \forall s'\in \mathcal{S} \\
\end{align*}

where
\[\tilde{P}_{t}^{\it UB}(s' \mid s, a) =
\begin{cases}
\dfrac{\min(P(s_{t+1} \mid s_t, a_t), P(s_{t+1} \mid s, a))}{P(s_{t+1} \mid s_t, a_t)} & \text{if $s' = s_{t+1}$}\\
0 & \text{if $P(s' \mid s_t, a_t) > 0$} \\ & \text{and $\dfrac{P(s_{t+1} \mid s, a)}{P(s_{t+1} \mid s_t, a_t)}\geq\dfrac{P(s' \mid s, a)}{P(s' \mid s_t, a_t)}$}\\
\min(P(s' \mid s, a), 1 - P(s_{t+1} \mid s, a)) & \text{if $P(s' \mid s_t, a_t) > 0$}\\
\min\left(1 - P(s_{t+1} \mid s, a), \dfrac{P(s' \mid s, a)}{P(s_{t+1} \mid s_t, a_t)}\right) & \text{otherwise}\\
\end{cases}
\]

This works, because we have proven in existing work that the counterfactual probability of any given transition depends only on the interventional distributions from the observed state-action pair $(s_t, a_t)$ and the counterfactual state-action pair $(s,a)$ (hence the closed-form solution for $\tilde{P}_{t}^{\it UB}(s' \mid s, a)$ depends only on these two distributions), so we don't need to consider any other state-action pair in our optimisation.

Given this nested optimisation, we can correctly calculate the lower bounds in Cases 1 and 3 as follows:

\[
\tilde{P}^{\it LB}_{\mathcal{M}_\updownarrow}(s_{t+1}\mid s, a) = \max\left(P^{\it LB}_{\mathcal{M}_\updownarrow}(s_{t+1} \mid s, a), 1-\max_{\substack{P_{\mathcal{M}}(\cdot \mid s_t, a_t) \in P_{\mathcal{M}_\updownarrow}(\cdot \mid s_t, a_t) \\ P_{\mathcal{M}}(\cdot \mid s,a) \in P_{\mathcal{M}_\updownarrow}(\cdot \mid s,a)}} \sum_{s'' \in \mathcal{S}\setminus\{s_{t+1}\}} {\tilde{P}_{t}^{\it UB}(s'' \mid s, a)}\right)
\]

\[
\tilde{P}^{\it LB}_{\mathcal{M}_\updownarrow}(s'\mid s, a) = \max\left(0, 1-\max_{\substack{P_{\mathcal{M}}(\cdot \mid s_t, a_t) \in P_{\mathcal{M}_\updownarrow}(\cdot \mid s_t, a_t) \\ P_{\mathcal{M}}(\cdot \mid s,a) \in P_{\mathcal{M}_\updownarrow}(\cdot \mid s,a)}} \sum_{s'' \in \mathcal{S}\setminus\{s'\}} {\tilde{P}_{t}^{\it UB}(s'' \mid s, a)}\right)
\]
\section{Gumbel-max SCMs}
\label{app: gumbel max}
The Gumbel-max SCM for an MDP is expressed as:
\begin{equation}\label{eq:gumbel-max-scm}
    S_{t+1} = f(S_{t},A_{t}, U_{t}=(G_{s,t})_{s\in \mathcal{S}}) 
    = \argmax_{s\in \mathcal{S}}\left\{\log\left(P_{\mathcal{M}}(s \mid S_t, A_t)\right) +G_{s,t}\right\}
\end{equation}
where $P_{\mathcal{M}}$ is the MDP's transition probabilities, $S_t$, $A_t$ and $S_{t+1}$ are endogenous variables representing the state, action and next state, and the values of the exogenous variable $U_{t}=(G_{s,t})_{s\in \mathcal{S}}$ are sampled from the standard Gumbel distribution\footnote{This formulation relies on the Gumbel-max trick \citep{maddison2014sampling} by which sampling from a categorical distribution with $n$ categories is equivalent to sampling $n$ values $g_1,\dots,g_n$ from the standard Gumbel distribution and evaluating $\argmax_{j}\left\{\log\left(P(Y=j)\right)+ g_j\right\}$ over all the categories, where $Y$ is the output.}. 

\begin{algorithm}[h]
\caption{Gumbel-max Topdown Sampling}
\label{alg:topdown}
\begin{algorithmic}
   \STATE {\bfseries Input:} $logits$, observed $s_{t+1}$, number of samples $nsamp=1$
   \STATE {\bfseries Output:} $gumbels$

   \STATE $topgumbel \gets$ sample $nsamp$ values from Gumbel(0,1)
   \STATE $gumbels \gets$ zero matrix of shape $(nsamp, |\mathcal{S}|)$

   \FOR{$i = 0$ {\bfseries to} $|\mathcal{S}| - 1$}
       \IF{$i = s_{t+1}$} 
           \STATE $gumbels[:,i] \gets topgumbel - logits[i]$ \COMMENT{Observed outcome}
       \ELSIF{$logits[i] \neq -\infty$}
           \STATE $g \gets$ sample $nsamp$ values from Gumbel(0,1) + logits[i]
           \STATE $gumbels[:,i] \gets -\log(\exp(-g) + \exp(-topgumbel)) - logits[i]$ \COMMENT{Truncated Gumbel}
       \ELSE
           \STATE $gumbels[:,i] \gets$ sample $nsamp$ values from Gumbel(0,1) \COMMENT{Unconstrained state}
       \ENDIF
   \ENDFOR

   \STATE {\bfseries return} $gumbels$
\end{algorithmic}
\end{algorithm}

We can perform (approximate) posterior inference of $P((G_{s,t})_{s\in \mathcal{S}} \mid s_t, a_t, s_{t+1})$ through rejection sampling \citep{oberst2019counterfactual}, or the top-down Gumbel sampling \citep{maddison2014sampling} algorithm, given in Algorithm \ref{alg:topdown}. This algorithm returns a finite sample of values from the posterior Gumbel distribution, $G'_{s, t}$. The counterfactual transition probabilities can then be calculated as follows:
\begin{multline}
\label{eq:cf_mdp_probs}
     P_{\mdp,t}(s' \mid s, a) = P\left(s' = \argmax_{q \in \states} \left\{\log\left(\probs(q \mid s, a)\right) + 
     G'_{q,t}\right\}\right)
     \approx \dfrac{1}{N} \sum_{j=1}^{N} \mathbbm{1}\left(s' = \argmax_{q \in \states} \left\{\log\left(\probs(q \mid s, a)\right) + 
     G_{q,t}^{\prime(j)}\right\}\right)
\end{multline}

Given an IMDP $\mathcal{M}_\updownarrow$, we can derive conservative bounds for counterfactual transition probabilities as follows;

\begin{theorem}\label{thm: gumbel bounds}
    Given an arbitrary IMDP $\mathcal{M}_\updownarrow$, observed transition $s_t, a_t \rightarrow s_{t+1}$, and counterfactual transition $s, a \rightarrow s'$, under the Gumbel-max SCM the counterfactual transition probabilities of the observed state-action pair are exact for all $\mathcal{M} \in \mathcal{M}_\updownarrow$: $\tilde{P}_t(s_{t+1} \mid s_t, a_t) = 1$ and $\tilde{P}_t(s' \mid s_t, a_t) = 0$ for all $s' \neq s_{t+1}$.
 
    For $(s, a) \neq (s_t, a_t)$, a lower bound for the counterfactual transition probability across all MDPs $\mathcal{M} \in \mathcal{M}_\updownarrow$, produced by the Gumbel-max SCM, is given by:
    \[
    \tilde{P}_t^{\it LB}(s' \mid s, a) = P\left(s' = \argmax_{q \in \states} \left\{\log\left(\hat{P}^{\it LB}(q \mid s, a)\right) + G'_{q,t}\right\}\right),
    \]
    \[
    \text{where} \quad \hat{P}^{\it LB}(q \mid s, a) =
    \begin{cases}
        P^{\it LB}_{\mathcal{M}_\updownarrow}(s' \mid s, a) & \text{if } q = s',\\
        P^{\it UB}_{\mathcal{M}_\updownarrow}(q \mid s, a) & \text{otherwise,}
    \end{cases}
    \]
    and $G'_{q,t}$ is derived given $P_{\mathcal{M}}(s' \mid s_t, a_t) = P^{\it UB}_{\mathcal{M}_\updownarrow}(s' \mid s_t, a_t)$, and $\forall s'' \neq s'$, $P_{\mathcal{M}}(s'' \mid s_t, a_t) = P^{\it LB}_{\mathcal{M}_\updownarrow}(s'' \mid s_t, a_t)$.
 
    An upper bound for the counterfactual transition probability across all MDPs $\mathcal{M} \in \mathcal{M}_\updownarrow$, produced by the Gumbel-max SCM, is given by:
    \[
    \tilde{P}_t^{\it UB}(s' \mid s, a) = P\left(s' = \argmax_{q \in \states} \left\{\log\left(\hat{P}^{\it UB}(q \mid s, a)\right) + G'_{q,t}\right\}\right),
    \]
    \[
    \text{where} \quad \hat{P}^{\it UB}(q \mid s, a) =
    \begin{cases}
        P^{\it UB}_{\mathcal{M}_\updownarrow}(s' \mid s, a) & \text{if } q = s',\\
        P^{\it LB}_{\mathcal{M}_\updownarrow}(q \mid s, a) & \text{otherwise,}
    \end{cases}
    \]
    and $G'_{q,t}$ is derived given $P_{\mathcal{M}}(s' \mid s_t, a_t) = P^{\it LB}_{\mathcal{M}_\updownarrow}(s' \mid s_t, a_t)$, and $\forall s'' \neq s'$, $P_{\mathcal{M}}(s'' \mid s_t, a_t) = P^{\it UB}_{\mathcal{M}_\updownarrow}(s'' \mid s_t, a_t)$.
\end{theorem}
\begin{proof}
We first consider the observed state-action pair $(s_t, a_t)$. For any $\mathcal{M} \in \mathcal{M}_\updownarrow$, its corresponding Gumbel-max SCM produces counterfactual probabilities $\tilde{P}_{\mdp,t}(s_{t+1} \mid s_t, a_t) = 1$ and $\tilde{P}_{\mdp,t}(s' \mid s_t, a_t) = 0$ for all $s' \neq s_{t+1}$, since the posterior Gumbel values are, by construction, those for which $s_{t+1}$ is the $\argmax$ under the same logits. These bounds are therefore exact.

For the remaining state-action pairs $(s,a) \neq (s_t, a_t)$, let us consider the lower counterfactual probability bound (the proof for the upper counterfactual probability bound follows similarly by reversing all bounds throughout). To derive the lower counterfactual probability bound, we need to consider how we can minimise $\tilde{P}_{\mdp,t}(s' \mid s, a)$ across all $\mdp \in \mdp_\updownarrow$:

\begin{equation}
{\min / \max}_{\mathcal{M} \in \mathcal{M_\updownarrow}} \tilde{P}_{\mdp,t}(s' \mid s, a) = {\min / \max}_{\mathcal{M} \in \mathcal{M_\updownarrow}} P\left(s' = \argmax_{q \in \states} \left\{\log\left(\probs(q \mid s, a)\right) + 
     G'_{q,t}\right\}\right)
\label{eq: gumbel-max optimisation}
\end{equation}

First, we assume the IMDP satisfies the normalisation assumption (Assumption~\ref{ass: normalisation}).

Next, from \eqref{eq: gumbel-max optimisation}, we can see that the counterfactual transition probabilities depend only on the logits of the observed state-action pair \((s_t, a_t)\) and the counterfactual state-action pair \((s, a)\). This means we only need to consider these logits within their bounds to determine the upper and lower bounds of the counterfactual transition probability for any transition. Also, from \eqref{eq: gumbel-max optimisation}, we can see that to minimise the counterfactual transition probability \(P_{\mdp,t}(s' \mid s, a)\), we must minimise the number of samples where \(s'\) is the \(\arg\max\) over all states \(q \in \mathcal{S}\) (and vice versa for maximising the counterfactual probability). To achieve this, we need to minimise both $\log P_\mathcal{M}(s' \mid s, a)$ and $G'_{s',t}$, while maximising the corresponding quantities $\log P_\mathcal{M}(q \mid s, a)$ and $G'_{q,t}$ of all other states $q \neq s'$.

Clearly, we can minimise $\log P_\mathcal{M}(s' \mid s, a)$ by setting $P_\mathcal{M}(s' \mid s, a) = P^{\it LB}_{\mathcal{M}_\updownarrow}(s' \mid s, a)$, and maximise $\log P_\mathcal{M}(q \mid s, a)$ for each $q \neq s'$ by setting $P_\mathcal{M}(q \mid s, a) = P^{\it UB}_{\mathcal{M}_\updownarrow}(q \mid s, a)$. Note that we take these extremal values for each next state individually, so together they need not form a valid probability distribution (i.e., they need not correspond to any admissible MDP $\mathcal{M} \in \mathcal{M}_\updownarrow$); as we show below, this still yields a valid, though possibly not tight, lower bound.

Now, we need to consider how to minimise the posterior Gumbel distribution \(G'_{s', t}\), which is derived from the top-down sampling algorithm in Algorithm~\ref{alg:topdown}. Note that the Gumbel values are determined entirely by the logits of the observed state-action pair, \(\log P_\mathcal{M}(\cdot \mid s_t, a_t)\). We will examine the following disjoint cases separately.

\paragraph{Case 1: \(s' = s_{t+1}\).} In this case, the Gumbel values are given by 
\[
\mathit{topgumbel} - \log P_\mathcal{M}(s_{t+1} \mid s_t, a_t),
\]
where \(\mathit{topgumbel}\) is sampled from the standard Gumbel distribution. Therefore, to minimise the Gumbel values, we need to maximise \(\log P_\mathcal{M}(s_{t+1} \mid s_t, a_t)\). Since \(\log P_\mathcal{M}(s_{t+1} \mid s_t, a_t) \leq 0\) (because \(0 \leq P_\mathcal{M}(s_{t+1} \mid s_t, a_t) \leq 1\)), this maximum is achieved when
\[
P_\mathcal{M}(s_{t+1} \mid s_t, a_t) = P_{\mathcal{M}_\updownarrow}^{\it UB}(s_{t+1} \mid s_t, a_t).
\]

\paragraph{Case 2: \(s' \neq s_{t+1}\) and \(P_{\mathcal{M}_\updownarrow}^{\it LB}(s' \mid s_t, a_t) > 0\).} 
If \(s' \neq s_{t+1}\) and \(P_{\mathcal{M}_\updownarrow}^{\it LB}(s' \mid s_t, a_t) > 0\), then the top-down sampling algorithm produces the following Gumbel values, across all MDPs \(\mathcal{M} \in \mathcal{M}_\updownarrow\):
\[
G'_{s', t} = -\log\bigl( \exp( - (g + \log P_\mathcal{M}(s' \mid s_t, a_t))) + \exp(-\mathit{topgumbel}) \bigr) - \log P_\mathcal{M}(s' \mid s_t, a_t),
\]
where \(g\) and \(\mathit{topgumbel}\) are sampled from the standard Gumbel distribution. 

Let us analyse how this depends on \(P_\mathcal{M}(s' \mid s_t, a_t)\). Set \(\log P_\mathcal{M}(s' \mid s_t, a_t) = L\), and define 
\[
Z = \exp( - (g + L)) + \exp( - \mathit{topgumbel}).
\]
Then
\[
G'_{s', t} = -\log Z - L.
\]
We compute
\begin{align*}
\frac{d}{dL} [ -\log Z - L ] 
&= \frac{d}{dL} [ -\log Z ] - 1 \\
&= -\frac{1}{Z} \frac{dZ}{dL} - 1,
\end{align*}
where
\begin{align*}
Z &= \exp( - (g + L)) + \exp( - \mathit{topgumbel}), \\
\frac{dZ}{dL} &= -\exp( - (g + L)).
\end{align*}
Thus
\begin{equation}
\begin{aligned}
\frac{d}{dL} &= -\frac{1}{Z} \cdot -\exp(-(g+L)) - 1\\
&= \frac{\exp(-(g+L))}{\exp(-(g + L)) + \exp(-topgumbel)} - 1\\
&= \frac{-\exp(-topgumbel)}{\exp(-(g + L)) + \exp(-topgumbel)}
\end{aligned}
\label{eq: derivative of gumbel-max}
\end{equation}
which is strictly negative, since \(\exp(x) > 0\) for all real \(x\).

Therefore, the minimum Gumbel value \(G'_{s', t}\) occurs when \(L\) is maximised, i.e.,
\[
P_\mathcal{M}(s' \mid s_t, a_t) = P_{\mathcal{M}_\updownarrow}^{\it UB}(s' \mid s_t, a_t)
\]

\paragraph{Case 3: \(s' \neq s_{t+1}\) and \(P_{\mathcal{M}_\updownarrow}^{\it UB}(s' \mid s_t, a_t) = 0\).} 
If \(P_{\mathcal{M}_\updownarrow}^{\it UB}(s' \mid s_t, a_t) = 0\), i.e., \(s'\) is not in the support of \((s_t, a_t)\), then the Gumbel values for \(s'\) are drawn from the standard Gumbel distribution. In this case, adjusting the logits for \(P_\mathcal{M}(\cdot \mid s_t, a_t)\) has no impact on the posterior Gumbel distribution, and therefore no effect on the counterfactual transition probability.

\paragraph{Case 4: \(s' \neq s_{t+1}\), \(P_{\mathcal{M}_\updownarrow}^{\it LB}(s' \mid s_t, a_t) = 0\), and \(P_{\mathcal{M}_\updownarrow}^{\it UB}(s' \mid s_t, a_t) > 0\).} 
If \(P_{\mathcal{M}_\updownarrow}^{\it LB}(s' \mid s_t, a_t) = 0\) and \(P_{\mathcal{M}_\updownarrow}^{\it UB}(s' \mid s_t, a_t) > 0\), then depending on the choice of MDP \(\mathcal{M} \in \mathcal{M}_\updownarrow\), the top-down sampling algorithm could fall into different cases. However, the truncated Gumbel values (which occur when \(P_\mathcal{M}(s' \mid s_t, a_t) > 0\)) are always smaller than Gumbel values sampled from the standard Gumbel distribution (which occur when \(P_\mathcal{M}(s' \mid s_t, a_t) = 0\)). Therefore, to minimise the lower counterfactual probability, we should select an MDP where \(P(s' \mid s_t, a_t) > 0\). Moreover, as shown in Case~2, the minimum Gumbel values for such an MDP occur when
\[
P_\mathcal{M}(s' \mid s_t, a_t) = P_{\mathcal{M}_\updownarrow}^{\it UB}(s' \mid s_t, a_t).
\]

Therefore, in all cases, we should derive $G'_{s',t}$ given $P_\mathcal{M}(s' \mid s_t, a_t) = P_{\mathcal{M}_\updownarrow}^{\it UB}(s' \mid s_t, a_t)$ to minimise $G'_{s',t}$; by the same argument (Case~2), the remaining posterior Gumbel values $G'_{q,t}$, $q \neq s'$, are maximised by setting $P_\mathcal{M}(q \mid s_t, a_t) = P_{\mathcal{M}_\updownarrow}^{\it LB}(q \mid s_t, a_t)$. As before, these extremal values are chosen for each next state individually and need not form a valid distribution.

However, since the derivation depends on the realised Gumbel samples, we cannot deterministically specify which state will be the $\argmax$ sample for a given choice of transition probabilities. But, for any fixed realisation of the Gumbel samples $g$ and $\mathit{topgumbel}$, whether $s'$ is the $\arg\max$ is monotone increasing in its interventional logit $\log P_{\mathcal M}(s' \mid s,a)$ and posterior Gumbel value $G'_{s',t}$, and monotone decreasing in those of every other state $q \neq s'$; hence so is the probability that $s'$ is the $\arg\max$. Therefore, although the $\arg\max$ outcome itself is random, we can minimise the probability that $s'$ is selected by minimising the former and maximising the latter.

Take the following assignment of interventional probabilities:

\begin{itemize}
\item $P_{\mathcal{M}}(s' \mid s, a) = P^{\it LB}_{\mathcal{M}_\updownarrow}(s' \mid s, a)$, and $\forall s'' \neq s', P_{\mathcal{M}}(s'' \mid s, a) = P^{\it UB}_{\mathcal{M}_\updownarrow}(s'' \mid s, a)$
\item $P_{\mathcal{M}}(s' \mid s_t, a_t) = P^{\it UB}_{\mathcal{M}_\updownarrow}(s' \mid s_t, a_t)$, and $\forall s'' \neq s', P_{\mathcal{M}}(s'' \mid s_t, a_t) = P^{\it LB}_{\mathcal{M}_\updownarrow}(s'' \mid s_t, a_t)$
\end{itemize}

If there exists an MDP $\mathcal{M} \in \mathcal{M}_\updownarrow$ with this assignment of probabilities, this assignment simultaneously minimises $P_{\mathcal M}(s' \mid s,a)$ and $G'_{s',t}$ and maximises those of all other states, and therefore minimises the probability that $s'$ will be the $\arg\max$, leading to a tight lower bound.

If no admissible MDP satisfies this assignment exactly, we may still evaluate the Gumbel-max expression at these extremal logits. Since the probability of the $\arg\max$ event is monotone in each logit (increasing in those of $s'$ and decreasing in those of all other states), evaluating it at these extremal (possibly infeasible) logits yields a conservative lower bound.

This proves that a conservative lower bound is:
\[
    \tilde{P}_t^{\it LB}(s' \mid s, a) = P\left(s' = \argmax_{q \in \states} \left\{\log\left(\hat{P}^{\it LB}(q \mid s, a)\right) + G'_{q,t}\right\}\right),
\]
where $\hat{P}^{\it LB}(q \mid s, a) = P^{\it LB}_{\mathcal{M}_\updownarrow}(s' \mid s, a)$ if $q = s'$ and $P^{\it UB}_{\mathcal{M}_\updownarrow}(q \mid s, a)$ otherwise, and $G'_{q,t}$ is derived given $P_\mathcal{M}(s' \mid s_t, a_t) = P_{\mathcal{M}_\updownarrow}^{\it UB}(s' \mid s_t, a_t)$ and $\forall s'' \neq s'$, $P_\mathcal{M}(s'' \mid s_t, a_t) = P_{\mathcal{M}_\updownarrow}^{\it LB}(s'' \mid s_t, a_t)$.

The proof for the upper bound follows similarly, by taking the opposite IMDP bounds throughout.
\end{proof}

In our experiments, we estimate these bounds using~\eqref{eq:cf_mdp_probs} with $N = 10{,}000$ samples from the posterior Gumbel distribution per transition. For the observed state-action pair $(s_t, a_t)$, we fix the counterfactual transition probabilities to $1$ for the observed next state $s_{t+1}$ and $0$ for all other states,  since under the Gumbel-max SCM the counterfactual of the observed state-action pair always reproduces $s_{t+1}$ (as in Theorem~\ref{theorem: observed imdp}). Because the remaining bounds are estimated by sampling, we clip them to $[0, 1]$ and, where sampling error places the lower bound above the upper bound, set the two equal. Finally, we rescale the lower bounds of each distribution to sum to at most 1 and the upper bounds to sum to at least 1, to ensure we have a valid CFIMDP.
\section{Confidence Intervals over Counterfactual IMDP}
\label{app: confidence proof}
Given a finite set of observed transitions in an MDP, we can build an interval MDP where transition probabilities are represented as intervals, capturing uncertainty in the transition probabilities due to limited data. For a given state-action pair, $(s, a)$, we can simulate a large number of transitions, and calculate:

\begin{itemize}
    \item $m(s,a)$: number of times action $a$ was taken in state $s$.
    \item $n(s,a,s')$: number of times the system transitioned to state $s'$ after being in state $s$ and taking action $a$.
    \item Empirical probability $\hat{P}(s' \mid s, a) = \frac{n(s, a, s')}{m(s, a)}$.
\end{itemize}

\subsection{Confidence Intervals}
There are several methods we could use to produce confidence intervals over the probability of each transition. \jl{We require all transition probabilities to lie within their intervals simultaneously with probability at least $1-\alpha$. To achieve this, we distribute the overall confidence evenly across all nondeterministic transitions: by a union bound, it suffices to compute each interval at confidence level $1-\alpha'$, where $\alpha' = \alpha / |\mathcal{T}'|$ and $\mathcal{T}'$ is the set of nondeterministic transitions\footnote{Since the topology of the MDP is known, deterministic transitions have probability exactly $1$ and require no confidence interval.}. There are several methods we could use to produce confidence intervals over the probability of each transition:
\begin{itemize}
    \item \textbf{Hoeffding's inequality} provides a high-probability bound on the deviation between the true transition probability $P(s' \mid s, a)$ and $\hat{P}(s' \mid s, a)$. For a confidence level of $1-\alpha'$,
        \[
        Pr(|P(s' \mid s, a)-\hat{P}(s' \mid s, a)| \geq \eta) \leq 2e^{(-2\cdot m(s, a) \cdot \eta^2)}
        \]
        where
        \[
        \eta = \sqrt{\frac{\ln(2/\alpha')}{2\cdot m(s, a)}}
        \]
        is chosen so that the right-hand side equals $\alpha'$. We can then construct an interval MDP as follows. For each transition $s, a \rightarrow s'$, we estimate its interval transition probability as:
        \[
        P_{\mathcal{M}_\updownarrow}(s' \mid s, a) = [\max(0, \hat{P}(s' \mid s, a) - \eta), \min(1, \hat{P}(s' \mid s, a) + \eta)]
        \]

    \item The \textbf{Clopper--Pearson} method provides an exact $(1-\alpha')$ confidence interval for a binomial proportion based on the cumulative distribution function of the binomial distribution. For a confidence level of $1-\alpha'$, the Clopper-Pearson interval is given by:
    \[
    P(s' \mid s,a) \in 
    \left[
    \mathrm{Beta}^{-1}\!\left(\frac{\alpha'}{2};\, n(s,a,s'),\, m(s,a)-n(s,a,s')+1\right),\;
    \mathrm{Beta}^{-1}\!\left(1-\frac{\alpha'}{2};\, n(s,a,s')+1,\, m(s,a)-n(s,a,s')\right)
    \right]
    \]
    where $\mathrm{Beta}^{-1}(q;\,a,b)$ denotes the $q$-quantile of the $\mathrm{Beta}(a,b)$ distribution. These intervals are typically tighter than those from Hoeffding's inequality, so we use the Clopper--Pearson method in all our experiments.
\end{itemize}}

\subsection{Construction and Normalisation of the IMDP}

In our approach, we use the Clopper–Pearson method to estimate the IMDP. The resulting intervals are then normalised to satisfy Assumptions~\ref{ass: normalisation} and~\ref{ass: observed transition}, which ensure that counterfactual probabilities can be computed over the IMDP.

We have constructed the IMDP such that for any transition $s, a \rightarrow s'$ in the MDP, \[Pr(P(s' \mid s, a) \notin P_{\mathcal{M}_\updownarrow}(s' \mid s, a))\leq \alpha\]

where $P_{\mathcal{M}_\updownarrow}(s' \mid s, a)$ is the true transition probability.

\subsection{Propagation of Confidence to the CFIMDP}
\begin{proof}[Proof of Proposition~\ref{prop: confidence}]
For now, assume the derived counterfactual probability bounds are tight (we will consider what happens if we apply our non-tight lower bound later). Take our IMDP $\mathcal{M}_\updownarrow$ derived from observational data using the Clopper--Pearson bounds. By construction, the chance of \jl{any true transition probability lying outside the IMDP is at most $\alpha$, i.e., $Pr\big(\exists\, s, a, s' : P(s' \mid s, a) \notin P_{\mathcal{M}_\updownarrow}(s' \mid s, a)\big) \leq \alpha$.}

\jl{Take an arbitrary transition $s, a \rightarrow s'$. In the worst case, where the true transition probability for this transition is \textbf{not} contained in the bound $[P_{\mathcal{M}_\updownarrow}^{\it LB}(s' \mid s, a), P_{\mathcal{M}_\updownarrow}^{\it UB}(s' \mid s, a)]$, the corresponding counterfactual probability bounds could be entirely disjoint from those in our CFIMDP (though in practice, some overlap is likely). However, the counterfactual bounds are exact given the true transition (interventional) probabilities, and, by construction, all true transition probabilities lie within the estimated IMDP with probability at least $1-\alpha$. Hence, with probability at least $1-\alpha$, all true counterfactual transition probabilities lie within their CFIMDP intervals simultaneously. Therefore, we can propagate the confidence bound through to the CFIMDP: the CFIMDP will contain the true counterfactual MDP (CFMDP) with probability at least $1-\alpha$.}

Finally, when we use a non-tight lower bound for the CFIMDP, this will be an over-approximation (which includes all valid CFMDPs and possibly some invalid ones). Therefore, this only increases our confidence that the computed counterfactual probability bounds contain the true counterfactual transition probability.
\end{proof}

\section{IMDP Value Iteration with Max Changes}
\label{app: algorithm}

\begin{algorithm}[h]
\caption{IMDP Value Iteration with Max Changes}\label{alg:vi_counter}
\begin{algorithmic}
   \STATE {\bfseries Input:} CFIMDP $\tilde{\mathcal M}^{\updownarrow}_\tau$ with bounds $\tilde P^{\updownarrow}_t$, observed path $\tau$, maximum number of changes $K$
   \STATE {\bfseries Output:} values $V^*$, policies $\pi^*_k$ for $k=0,\ldots,K$

   \FOR{$s \in \mathcal S$, $k = 0$ {\bfseries to} $K$}
       \STATE $V^*(s,|\tau|,k) \gets 0$
   \ENDFOR
   \FOR{$t = |\tau|-1$ {\bfseries down to} $0$}
       \FOR{$k = 0$ {\bfseries to} $K$}
           \FOR{$s \in \mathcal S$}
               \STATE $best \gets -\infty$
               \FOR{$a \in \mathcal A$}
                   \IF{$k = 0$ {\bfseries and} $a \neq a_t$}
                       \STATE {\bfseries continue} \COMMENT{no changes remaining}
                   \ENDIF
                   \STATE $k' \gets k - \mathbbm{1}\{a \neq a_t\}$
                   \STATE $Q \gets R(s,a) + \min_{\tilde P_t \in \tilde P^{\updownarrow}_t} \sum_{s'} \tilde P_t(s' \mid s,a)\, V^*(s',t+1,k')$
                   \IF{$Q > best$}
                       \STATE $best \gets Q$; \quad $\pi^*_k(t,s) \gets a$
                   \ENDIF
               \ENDFOR
               \STATE $V^*(s,t,k) \gets best$
           \ENDFOR
       \ENDFOR
   \ENDFOR
   \STATE {\bfseries return} $V^*$, $\{\pi^*_k\}_{k=0}^{K}$
\end{algorithmic}
\end{algorithm}

\section{Environment Descriptions and System Details}
\label{app: environments}
Below, we provide brief descriptions of the MDP environments and system specifications used in our experiments. The same environments as \citet{lally2025robustcounterfactualinferencemarkov} are used to allow for comparison.

\subsection{GridWorld}
In the GridWorld MDP, an agent navigates a $4\times4$ grid starting from the top-left corner with the objective of reaching the goal state in the bottom-right corner. At each time step, the agent can choose to move up, down, left, or right; however, it moves only in the chosen direction with probability $p$, otherwise moving in a random direction. Reaching the goal state yields a reward of $100$, and intermediate transitions provide shaped rewards proportional to the agent's proximity to the goal. The grid also contains a hazardous, terminal state in the centre, and entering it incurs a penalty of $-100$. In our experiments, we use two variants of the environment with different values of $p$, $p=0.9$ (i.e., $90$\% chance of moving in the chosen direction) and $p=0.4$.

\subsection{Frozen Lake}
The Frozen Lake MDP is adapted from PRISM \citep{prism, cool-mc}. In this environment, the agent navigates a $4\times4$ grid to retrieve a frisbee, while avoiding a set of holes scattered across the grid. Since the original PRISM environment is formulated as a reachability problem, the environment has been converted to a reward-based setting, which yields a reward of $100$ for successfully reaching the frisbee, and a penalty of $-100$ for falling into a hole. Both the goal state and states with holes are terminal.

\subsection{Sepsis}
The Sepsis MDP \citep{oberst2019counterfactual} simulates patient trajectories in a sepsis treatment setting. At each time step, the patient's state is defined by four vital signs (heart rate, blood pressure, oxygen concentration, and glucose levels), each categorised as low, in normal range, or high. The RL agent controls three treatments (which can be switched on or off), with the goal of discharging the patient (which occurs when at least three of the vital signs are in normal range, and all three treatments are switched off). The patient dies if at least three vital signs are out of normal range at any time step. Rewards are shaped by the number of vital signs that are out of normal range, ranging from $-1000$ for patient death to $+1000$ for patient discharge.

\subsection{Aircraft}
The Aircraft MDP is a navigation task, also adapted from PRISM \citep{prism, cool-mc}. In this environment, the RL agent controls an aircraft that advances forward by one unit at each time step, and must choose to increase, decrease, or maintain its altitude, with a fixed probability of success. The goal is to successfully reach the right-hand side of the grid without colliding with another aircraft, which independently increases, decreases, or maintains its current altitude at each time step. Since the original PRISM environment is formulated as a reachability problem, it has been converted to a reward-based setting, which yields a reward of $+100$ for reaching the goal or a penalty of $-100$ for colliding with the other aircraft. Intermediate transitions receive a shaped reward inversely proportional to the proximity of the two aircraft.

\subsection{System Specification}
Our experiments were run in parallel using $32$ threads on a $128$-core machine with an Intel Xeon CPU and $512$ GB RAM. Simulation of IMDPs and construction of the corresponding CFIMDPs were implemented in Python 3.10, while our extended value iteration algorithm was implemented by extending the Julia Interval MDP package (v0.6.0) with the algorithm described in Appendix \ref{app: algorithm}.

\section{Additional Experiments}
\label{app: additional experiments}
\subsection{Correctness for Off-policy Evaluation Experiment: GridWorld ($p=0.4$)}
\begin{table}[H]
\centering
  \caption{Given $100$ randomly sampled observed trajectories, we compare the average lower and upper bounds, $V^{\bot,\pi^*}(s_0)$ and $V^{\top,\pi^*}(s_0)$, on the counterfactual value under the exact CFIMDP with those under data-driven CFIMDPs, whose IMDPs are learnt from $N$ simulated trajectories. Averaged over observed trajectories, the counterfactual value equals the interventional value, so every pair of bounds should contain the true interventional value $V^{\pi^*}(s_0)=-40.1$.}%
\label{tab:cf_correctness_simulated_gridworld_0.4}

\begin{tabular}{lcc}
\toprule
& \textbf{Lower Bound} & \textbf{Upper Bound} \\
\midrule
Exact Bounds & $-192 \pm 163$ & $102 \pm 154$ \\
\midrule
\multicolumn{3}{l}{\textbf{Data-driven Bounds}} \\
$N=10{,}000$       & $-284 \pm 168$    & $144 \pm 143$ \\
$N=100{,}000$      & $-232 \pm 164$    & $125 \pm 136$ \\
$N=1{,}000{,}000$  & $-221 \pm 162$    & $119 \pm 138$ \\
\bottomrule
\end{tabular}
\end{table}
The results for the GridWorld ($p=0.4$) experiment closely mirror those of the GridWorld ($p=0.9$) experiment in Section \ref{sec: correctness}, further confirming that our approach is valid and unbiased, as the data-driven CFIMDP bounds again consistently contain the true expected return, and converge towards the exact bounds as the number of trajectories increases.

\subsection{Computing Minimal Robust Counterfactual Interventions: Sepsis}
Figure~\ref{fig: sepsis cf paths} gives a more detailed view of the Sepsis counterfactual paths summarised in Figure~\ref{fig: counterfactual visualisation}. For the same catastrophic observed trajectory and each $k$, we plot $200$ counterfactual paths sampled under the policy $\pi^*_k$ and an adversarially chosen CFMDP within the exact CFIMDP. Each row shows one component of the patient's state over time: heart rate, systolic blood pressure, oxygen saturation and glucose, followed by the three treatments (antibiotics, vasopressors and ventilation). With $k=0$, the counterfactual paths follow the observed treatments and the patient's vital signs deteriorate as observed. As $k$ increases, $\pi^*_k$ may change more of the observed treatments, and the vital signs of more counterfactual patients return to the normal range (dotted lines), consistent with the rise in recovery shown in Figure~\ref{fig: counterfactual visualisation}.
\begin{figure}[H]
  \centering
  \includegraphics[width=\textwidth]{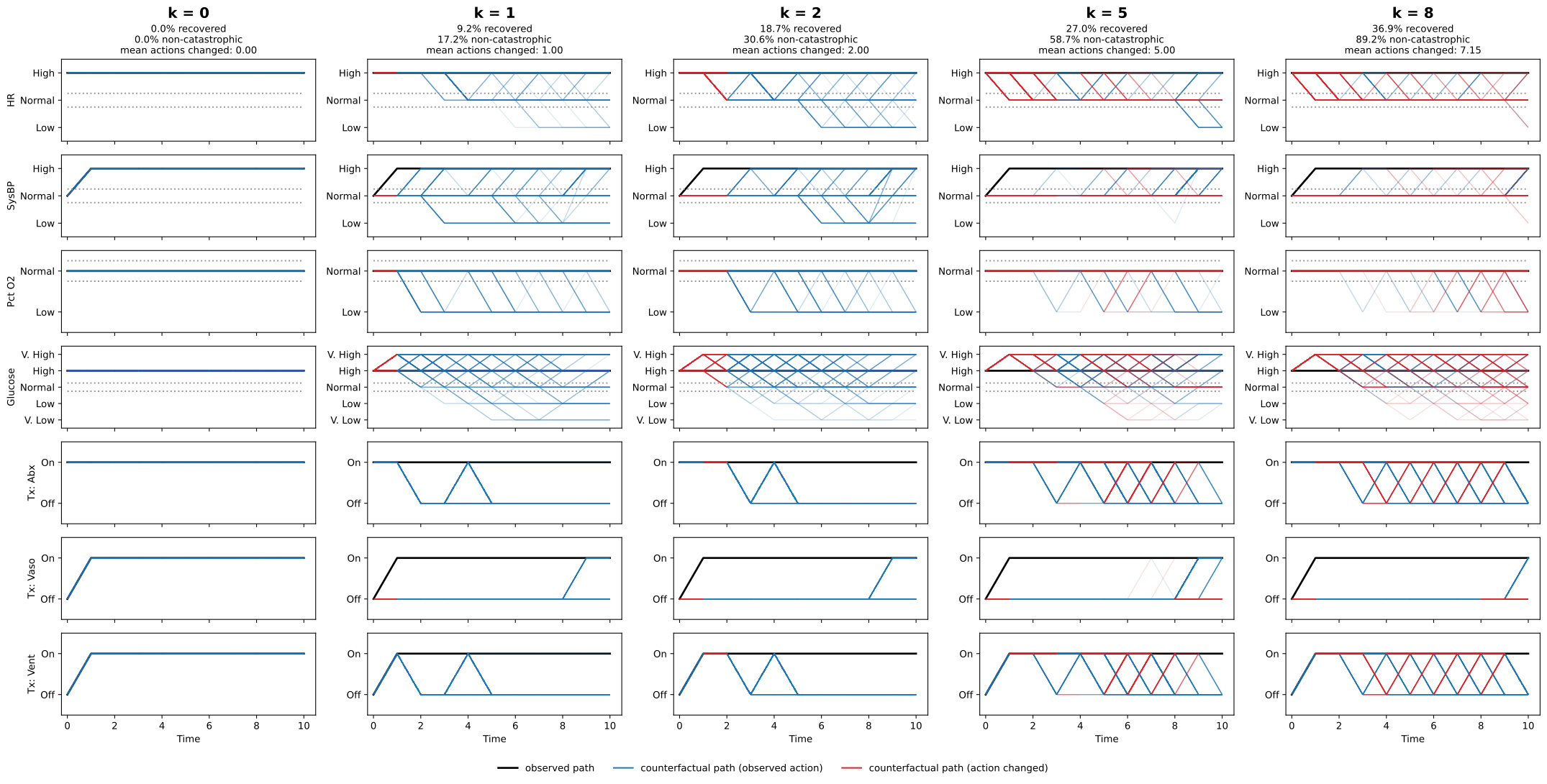}
\caption{Counterfactual Sepsis trajectories for one catastrophic observed trajectory, under $\pi^*_k$ for varying numbers of allowed action changes $k$. Rows show the patient's vital signs and applied treatments over time. The observed path is shown in black, and the $200$ sampled counterfactual paths in blue, or red at steps where an action differs from the observed one (in the treatment rows, where that treatment was changed). Dotted lines mark the normal range of each vital sign.}
  \label{fig: sepsis cf paths}
\end{figure}

\subsection{Execution Times and Bound Widths}
\begin{table}[H]
\centering
\caption{Left: effect on CFIMDP transitions of using tight (but more expensive) bounds in place of conservative closed-form bounds. Right: comparison of runtimes for constructing the CFIMDP (using $N=100{,}000, \alpha=0.1$ and averaging over $100$ randomly generated suboptimal paths).}
\resizebox{\linewidth}{!}{
\begin{tabular}{lccccccccc}
\toprule
& \multicolumn{4}{c}{\textbf{Bound Width Statistics}} 
& \multicolumn{4}{c}{\textbf{Execution Time (s)}} \\
\cmidrule(lr){2-5} \cmidrule(lr){6-9}

\textbf{Environment}
& \multicolumn{2}{c}{\textbf{Transitions Affected (\%)}} 
& \multicolumn{2}{c}{\textbf{Avg Difference in Width}}
& \textbf{Exact CFIMDP}
& \makecell{\textbf{Data-driven CFIMDP} \\ \textbf{(Cons. LB)}}
& \makecell{\textbf{Data-driven CFIMDP} \\ \textbf{(Exact LB)}}
& \makecell{\textbf{Data-driven CFIMDP} \\ \textbf{(Gumbel-max)}} \\

\cmidrule(lr){2-3} \cmidrule(lr){4-5}

& \textbf{All} & \textbf{Non-zero}
& \textbf{All} & \textbf{Non-zero}
& & & & \\

\midrule

GridWorld ($p=0.9$) 
& $4.37 \pm 0.46$ 
& $20.6 \pm 2.1$ 
& $4.89\times10^{-4} \pm 7.6\times10^{-5}$
& $0.0112 \pm 0.0013$ 
& $1.67 \pm 0.08$
& $1.71 \pm 0.09$
& $2.64 \pm 0.13$
& $1.55 \pm 0.07$ \\

GridWorld ($p=0.4$) 
& $1.14 \pm 0.74$ & $5.5 \pm 3.5$ 
& $2.8\times10^{-4} \pm 3.9\times10^{-4}$ & $0.018 \pm 0.014$ 
& $1.62 \pm 0.12$
& $1.65 \pm 0.06$
& $2.40 \pm 0.11$
& $1.55 \pm 0.10$ \\

Sepsis
& $0.0041 \pm 0.0050$ & $1.0 \pm 1.2$
& $3.9\times10^{-7} \pm 3.3\times10^{-7}$ & $0.024\pm 0.013$
& $57.2 \pm 3.9$
& $141 \pm 31$
& $1430 \pm 580$
& --$^{\dagger}$ \\

Frozen Lake
& $0.19 \pm 0.12$ & $1.49 \pm 0.91$ 
& $1.5\times10^{-5} \pm 1.2\times10^{-5}$ & $0.0070 \pm 0.0029$ 
& $2.24 \pm 0.11$
& $2.34 \pm 0.11$
& $3.75 \pm 0.30$
& $2.03 \pm 0.10$ \\

Aircraft
& $0.0097 \pm 0.0052$ & $0.90 \pm 0.48$
& $5.7\times10^{-6} \pm 3.3\times10^{-6}$ & $0.059 \pm 0.011$ 
& $7.9 \pm 1.8$
& $10.0 \pm 1.9$
& $85 \pm 13$
& $293 \pm 7$ \\
\bottomrule
\end{tabular}
}
\vspace{2pt}
\begin{minipage}{\linewidth}
\footnotesize $^{\dagger}$We omit Sepsis, as generating its Gumbel-max CFIMDPs is computationally prohibitive.
\end{minipage}
\label{tab:combined_results}
\end{table}

\end{document}